\documentclass{article}

\PassOptionsToPackage{numbers, compress}{natbib}
\usepackage[final]{neurips_2019}

\usepackage[utf8]{inputenc} % allow utf-8 input
\usepackage[T1]{fontenc}    % use 8-bit T1 fonts
\usepackage{hyperref}       % hyperlinks
\usepackage{url}            % simple URL typesetting
\usepackage{booktabs}       % professional-quality tables
\usepackage{amsfonts}       % blackboard math symbols
\usepackage{nicefrac}       % compact symbols for 1/2, etc.
\usepackage{microtype}      % microtypography
\usepackage{color}
\usepackage{amsmath,amsthm,amssymb}
\usepackage{algorithm}
\usepackage{algorithmicx}
\usepackage{graphicx}
\usepackage{subcaption}
\usepackage[noend]{algpseudocode}
\usepackage{enumitem}
\usepackage{csquotes}

\newtheorem{lemma}{Lemma}
\newtheorem{theorem}{Theorem}

\newtheorem{definition}{Definition}

\newcommand{\scE}{\mathcal{E}}
\newcommand{\scN}{\mathcal{N}}
\newcommand{\scO}{\mathcal{O}}
\newcommand{\scT}{\mathcal{T}}

\DeclareMathOperator*{\argmin}{argmin}

\newcommand{\field}[1]{\mathbb{#1}}
\newcommand{\R}{\field{R}}
\newcommand{\Nat}{\field{N}}
\newcommand{\E}{\field{E}}
\renewcommand{\Pr}{\field{P}}
\newcommand{\ceil}[1]{\lceil{#1}\rceil}

\newcommand{\theset}[2]{ \left\{ {#1} \,:\, {#2} \right\} }
\newcommand{\Ind}[1]{ \field{I}\left\{{#1}\right\} }

\newcommand{\spin}{\{-1,+1\}}
\newcommand{\OPT}{\mathrm{OPT}}
\newcommand{\access}{\ensuremath{\mathrm{ACC}}}
\newcommand{\ermcc}{\ensuremath{\mathrm{ERMCC}}}
\newcommand{\Vsing}{V_{\mathrm{sing}}}
\newcommand{\Sfin}{S_{\mathrm{fin}}}
\newcommand{\Vfin}{V_{\mathrm{fin}}}
\newcommand{\dfin}{d_{\mathrm{fin}}}
\newcommand{\nfin}{n_{\mathrm{fin}}}
\newcommand{\ve}{\varepsilon}

\renewcommand{\hat}{\widehat}
\renewcommand{\bar}{\overline}
\renewcommand{\epsilon}{\ve}

\title{
Correlation Clustering \\ with Adaptive Similarity Queries
}

\newif\ifproofs
\proofsfalse

\author{%
  Marco Bressan \\
  Department of Computer Science \\
  University of Rome Sapienza
  \And
  Nicolò Cesa-Bianchi \\
  Department of Computer Science \& DSRC \\
  Università degli Studi di Milano \\
  \AND
  Andrea Paudice \\
  Department of Computer Science\\
  Università degli Studi di Milano \& IIT \\
  \And
  Fabio Vitale \\ 
  Department of Computer Science\\
  University of Lille \& Inria \\
}

\begin{document}

\newcommand{\cs}{\mathcal{C}}
\newcommand{\Hn}{\mathcal{H}_n}
\newcommand{\cut}{\operatorname{cut}}
\newcommand{\prob}{\Pr}
\newcommand{\aggress}{\ensuremath{\operatorname{ACCESS}}}
\newcommand{\kc}{KwikCluster}

\maketitle

\begin{abstract}
In correlation clustering, we are given $n$ objects together with a binary similarity score between each pair of them.
The goal is to partition the objects into clusters so to minimise the disagreements with the scores.
In this work we investigate correlation clustering as an active learning problem: each similarity score can be learned by making a query, and the goal is to minimise both the disagreements and the total number of queries.
On the one hand, we describe simple active learning algorithms, which provably achieve an almost optimal trade-off while giving cluster recovery guarantees, and we test them on different datasets.
On the other hand, we prove information-theoretical bounds on the number of queries necessary to guarantee a prescribed disagreement bound.
These results give a rich characterization of the trade-off between queries and clustering error.
\end{abstract}

\section{Introduction}
\label{s:intro}
Clustering is a central problem in unsupervised learning. A clustering problem is typically represented by a set of elements together with a notion of similarity (or dissimilarity) between them. When the elements are points in a metric space, dissimilarity can be measured via a distance function. In more general settings, when the elements to be clustered are members of an abstract set $V$, similarity is defined by an arbitrary symmetric function $\sigma$ defined on pairs of distinct elements in $V$. Correlation Clustering (CC) \cite{bansal2004correlation} is a well-known special case where $\sigma$ is a $\spin$-valued function establishing whether any two distinct elements of $V$ are similar or not. The objective of CC is to cluster the points in $V$ so to maximize the correlation with $\sigma$. More precisely, CC seeks a clustering minimizing the number of errors, where an error is given by any pair of elements having similarity $-1$ and belonging to the same cluster, or having similarity $+1$ and belonging to different clusters. Importantly, there are no a priori limitations on the number of clusters or their sizes: all partitions of $V$, including the trivial ones, are valid. Given $V$ and $\sigma$, the error achieved by an optimal clustering is known as the \textsl{Correlation Clustering index}, denoted by $\OPT$. 
A convenient way of representing $\sigma$ is through a graph $G = (V,E)$ where $\{u,v\}\in E$ iff $\sigma(u,v) = +1$. Note that $\OPT=0$ is equivalent to a perfectly clusterable graph (i.e., $G$ is the union of disjoint cliques). Since its introduction, CC has attracted a lot of interest in the machine learning community, and has found numerous applications in  entity resolution \cite{getoor2012entity}, image analysis \cite{kim2011higher}, and social media analysis \cite{tang2016survey}. Known problems in data integration \cite{cohen2002learning} and biology \cite{ben1999clustering} can be cast into the framework of CC \cite{Sammut:2010}.

From a machine learning viewpoint, we are interested in settings when the similarity function $\sigma$ is not available beforehand, and the algorithm must learn $\sigma$ by querying for its value on pairs of objects. This setting is motivated by scenarios in which the similarity information is costly to obtain.
For example, in entity resolution, disambiguating between two entities may require invoking the user's help.
Similarly, deciding if two documents are similar may require a complex computation, and possibly the interaction with human experts.
In these active learning settings, the learner's goal is to trade the clustering error against the number of queries.
Hence, the fundamental question is: how many queries are needed to achieve a specified clustering error?
Or, in other terms, how close can we get to $\OPT$, under a prescribed query budget $Q$?

\subsection{Our Contributions}
In this work we characterize the trade-off between the number $Q$ of queries and the clustering error on $n$ points.
The table below here summarizes our bounds in the context of previous work. Running time and upper/lower bounds on the expected clustering error are expressed in terms of the number of queries $Q$, and all our upper bounds assume $Q = \Omega(n)$ while our lower bounds assume $Q = \scO(n^2)$.

\begin{center}
\begin{tabular}{lcc}
\toprule
Running time & Expected clustering error & Reference \\
\midrule
$Q$ + LP solver + rounding & $3(\ln n + 1)\OPT + \scO\big(n^{5/2}/\sqrt{Q}\big)$ & \cite{cesa2012correlation} \\
$Q$ & $3\OPT + \scO(n^3/Q)$ & Theorem~\ref{thm:access_cost} (see also \cite{bonchi2013local}) \\
Exponential & $\OPT + \scO\big(n^{5/2}/\sqrt{Q}\big)$ & Theorem~\ref{thm:fa_scheme} \\
Exponential ($\OPT=0$) & $\widetilde{\scO}\big(n^3/Q\big)$ & Theorem~\ref{thm:fa_scheme} \\\midrule
Unrestricted ($\OPT=0$) & $\Omega\big(n^2/\sqrt{Q}\big)$ & Theorem~\ref{th:lower} \\
Unrestricted ($\OPT \gg 0$)& $\OPT + \Omega\big(n^3/Q\big)$ & Theorem~\ref{thm:new_LB} \\
\bottomrule
\end{tabular}
\end{center}

Our first set of contributions is algorithmic.
We take inspiration from an existing greedy algorithm, \kc~\cite{Ailon2008}, that has expected error $3\OPT$ but a vacuous $\scO(n^2)$ worst-case bound on the number of queries.
We propose a variant of \kc, called \access, for which we prove several desirable properties.
First, \access\ achieves expected clustering error $3\OPT + \scO(n^3/Q)$, where $Q = \Omega(n)$ is a deterministic bound on the number of queries. In particular, if \access\ is run with $Q = \binom{n}{2}$, then it becomes exactly equivalent to \kc.
Second, \access\ recovers adversarially perturbed latent clusters. % $C$.
More precisely, if the input contains a cluster $C$ obtained from a clique by adversarially perturbing a fraction $\epsilon$ of its edges (internal to the clique or leaving the clique), then \access\ returns a cluster $\hat{C}$ such that $\E\big[|C\oplus \hat{C}|\big] = \scO\big(\ve|C| + n^2/Q\big)$, where $\oplus$ denotes symmetric difference.
This means that \access\ recovers almost completely all perturbed clusters that are large enough to be ``seen'' with $Q$ queries.
We also show, under stronger assumptions, that via independent executions of \access\ one can recover exactly all large clusters with high probability.
Third, we show a variant of \access, called \aggress\ (for Early Stopping Strategy), that makes significantly less queries on some graphs.
For example, when $\OPT=0$ and there are $\Omega\big(n^3/Q\big)$ similar pairs, the expected number of queries made by \aggress\ is only the square root of the queries made by \access.
In exchange, \aggress\ makes at most $Q$ queries in expectation rather than deterministically.

Our second set of contributions is a nearly complete information-theoretic characterization of the query vs.\ clustering error trade-off (thus, ignoring computational efficiency).
Using VC theory, we prove that for all $Q = \Omega(n)$ the strategy of minimizing disagreements on a random subset of pairs achieves, with high probability, clustering error bounded by $\OPT + \scO\big(n^{5/2}/\sqrt{Q}\big)$, which reduces to $\widetilde{\scO}\big(n^3/Q\big)$ when $\OPT=0$.
The VC theory approach can be applied to any efficient approximation algorithm, too.
The catch is that the approximation algorithm cannot ask the similarity of arbitrary pairs, but only of pairs included in the random sample of edges.
The best known approximation factor in this case is $3(\ln n + 1)$ \cite{demaine2006correlation}, which gives a clustering error bound of $3(\ln n + 1)\OPT + \scO\big(n^{5/2}/\sqrt{Q}\big)$ with high probability.
This was already observed in \cite{cesa2012correlation} albeit in a slightly different context. 

We complement our upper bounds by developing two information-theoretic lower bounds; these lower bounds apply to any algorithm issuing $Q = \scO(n^2)$ queries, possibly chosen in an adaptive way.
For the general case, we show that any algorithm must suffer an expected clustering error of at least $\OPT + \Omega\big(n^3/Q\big)$. In particular, for $Q = \Theta(n^2)$ any algorithm still suffers an additive error of order $n$,
and for $Q=\Omega(n)$ our algorithm \access\ is essentially optimal in its additive error term.
For the special case $\OPT=0$, we show a lower bound $\Omega\big(n^2/\sqrt{Q}\big)$. 

Finally, we evaluate our algorithms empirically on real-world and synthetic datasets.

\section{Related work}
Minimizing the correlation clustering error is APX-hard \cite{charikar2005clustering}, and the best efficient algorithm found so far achieves $2.06\,\OPT$ \cite{Chawla:2015}.
This almost matches the best possible approximation factor $2$ achievable via the natural LP relaxation of the problem~\cite{charikar2005clustering}.
A very simple and elegant algorithm for approximating CC is \kc\ \cite{Ailon2008}.
At each round, \kc\ draws a random pivot $\pi_r$ from $V$, queries the similarities between $\pi_r$ and every other node in $V$, and creates a cluster $C$ containing $\pi_r$ and all points $u$ such that $\sigma(\pi_r,u) = +1$.
The algorithm then recursively invokes itself on $V\setminus C$.
On any instance of CC, \kc\ achieves an expected error bounded by $3\OPT$.
However, it is easy to see that \kc\ makes $\Theta(n^2)$ queries in the worst case (e.g., if $\sigma$ is the constant function $-1$).
Our algorithms can be seen as a parsimonious version of \kc\ whose goal is reducing the number of queries.

The work closest to ours is \cite{bonchi2013local}.
Their algorithm runs \kc\ on a random subset of $1/(2\ve)$ nodes and stores the set $\Pi$ of resulting pivots.
Then, each node $v\in V\setminus\Pi$ is assigned to the cluster identified by the pivot $\pi\in\Pi$ with smallest index and such that $\sigma(v,\pi) = +1$.
If no such pivot is found, then $v$ becomes a singleton cluster.
According to \cite[Lemma~4.1]{bonchi2013local}, the expected clustering error for this variant is $3\OPT + \scO\big(\epsilon n^2\big)$, which can be compared to our bound for \access\ by setting $Q=n/\epsilon$.
On the other hand our algorithms are much simpler and significantly easier to analyze.
This allows us to prove a set of additional properties, such as cluster recovery and instance-dependent query bounds.
It is unclear whether these results are obtainable with the techniques of \cite{bonchi2013local}.

Another line of work attempts to circumvent computational hardness by using the more powerful same-cluster queries (SCQ). A same-cluster query tells whether any two given nodes are clustered together according to an optimal clustering or not. In \cite{DBLP:conf/latin/AilonBJ18} SCQs are used to design a FPTAS for a variant of CC with bounded number of clusters. In \cite{DBLP:conf/esa/SahaS19} SCQs are used to design algorithms for solving CC optimally by giving bounds on $Q$ which depend on $\OPT$.
Unlike our setting, both works assume \textsl{all} $\binom{n}{2}$ similarities are known in advance.
The work \cite{mazumdar2017clustering} considers the case in which there is a latent clustering with $\OPT=0$. 
The algorithm can issue SCQs, however the oracle is noisy: each query is answered incorrectly with some probability, and the noise is persistent (repeated queries give the same noisy answer). The above setting is closely related to the stochastic block model (SBM), which is a well-studied model for cluster recovery \cite{abbe2015community,massoulie2014community,mossel2018proof}. However, few works investigate SBMs with pairwise queries \cite{chen2016community}. Our setting is strictly harder because our oracle has a budget of $\OPT$ adversarially incorrect answers.

A different model is edge classification. Here the algorithm is given a graph $\mathcal{G}$ with hidden binary labels on the edges. The task is to predict the sign of all edges by querying as few labels as possible \cite{cesa2012correlation,chen2014clustering,chiang2014prediction}. As before, the oracle can have a budget $\OPT$ of incorrect answers, or a latent clustering with $\OPT=0$ is assumed and the oracle's answers are affected by persistent noise. Unlike correlation clustering, in edge classification the algorithm is not constrained to predict in agreement with a partition of the nodes. On the other hand, the algorithm cannot query arbitrary pairs of nodes in $V$, but only those that form an edge in $\mathcal{G}$.

\paragraph{Preliminaries and notation.}
\label{s:prel}
We denote by $V \equiv \{1,\ldots,n\}$ the set of input nodes, by $\scE\equiv\binom{V}{2}$ the set of all pairs $\{u,v\}$ of distincts nodes in $V$, and by $\sigma : \scE\to\spin$ the binary similarity function.
A clustering $\cs$ is a partition of $V$ in disjoint clusters $C_i:i=1,\ldots,k$.
Given $\cs$ and $\sigma$, the set $\Gamma_\cs$ of mistaken edges contains all pairs $\{u,v\}$ such that $\sigma(u,v)=-1$ and $u,v$ belong to same cluster of $\cs$ and all pairs $\{u,v\}$ such that $\sigma(u,v)=+1$ and $u,v$ belong to different clusters of $\cs$.
The cost $\Delta_\cs$ of $\cs$ is $\big|\Gamma_\cs\big|$.
The correlation clustering index is $\OPT = \min_\cs\Delta_\cs$, where the minimum is over all clusterings $\cs$.
We often view $V,\sigma$ as a graph $G=(V,E)$ where $\{u,v\} \in E$ is an edge if and only if $\sigma(u,v)=+1$.
In this case, for any subset $U \subseteq V$ we let $G[U]$ be the subgraph of $G$ induced by $U$, and for any $v\in V$ we let $\scN_v$ be the neighbor set of $v$.

A triangle is any unordered triple $T = \{u,v,w\} \subseteq V$.
We denote by $e = \{u,w\}$ a generic triangle edge; we write $e \subset T$ and $v \in T\setminus e$.
We say $T$ is a \emph{bad triangle} if the labels $\sigma(u,v), \sigma(u,w), \sigma(v,w)$ are $\{+, +, -\}$ (the order is irrelevant).
We denote by $\scT$ the set of all bad triangles in $V$.
It is easy to see that the number of edge-disjoint bad triangles is a lower bound on $\OPT$.

Due to space limitations, here most of our results are stated without proof, or with a concise proof sketch; the full proofs can be found in the supplementary material.

\section{The \access{} algorithm}
\label{s:algo}
We introduce our active learning algorithm $\access$ (Active Correlation Clustering).
\renewcommand{\algorithmicrequire}{\textbf{Parameters:}}
\begin{algorithm}[H]
\caption{
\label{alg:access}
%Invoked as $\access(V_1,1)$ where $V_1 \equiv V$ and $r=1$ is the index of the recursive call.
\access\ with query rate $f$
}
\begin{algorithmic}[1]
\setcounter{ALG@line}{0}
\Require{
residual node set $V_r$, round index $r$
}
\If{$|V_r|=0$} RETURN
\label{line:stop1}
\EndIf
\If{$|V_r| =1 $} output singleton cluster $V_r$ and RETURN
\label{line:stop2}
\EndIf
\If{$r > \lceil f(|V_1|-1) \rceil$}
RETURN
\EndIf
\State Draw pivot $\pi_r$ u.a.r.\ from $V_r$
\State $C_r \gets \{\pi_r\}$ \Comment{Create new cluster and add the pivot to it}
\State Draw a random subset $S_r$ of $\lceil f(|V_r|-1) \rceil$ nodes from $V_r\setminus\{\pi_r\}$ \label{line:queries1}
\For{each $u \in S_r$} query $\sigma(\pi_r,u)$ \label{line:queries2}
\EndFor
\If{$\exists\,u \in S_r$ such that $\sigma(\pi_r,u)=+1$} \label{line:sing} \Comment{Check if there is at least a positive edge} 
    \State Query all remaining pairs $(\pi_r,u)$ for $u \in V_r\setminus\big(\{\pi_r\}\cup S_r\big)$
    \State $C_r \gets C_r \cup \theset{u}{\sigma(\pi_r,u) = +1}$  \Comment{Populate cluster based on queries}
\EndIf
\State Output cluster $C_r$
\State $\access(V_r \setminus C_r, r+1)$ \Comment{Recursive call on the remaining nodes}
\end{algorithmic}
\end{algorithm}

$\access$ has the same recursive structure as \kc.
First, it starts with the full instance $V_1=V$.
Then, for each round $r=1,2,\ldots$ it selects a random pivot $\pi_r \in V_r$, queries the similarities between $\pi_r$ and a subset of $V_r$, removes $\pi_r$ and possibly other points from $V_r$, and proceeds on the remaining residual subset $V_{r+1}$. 
However, while \kc\ queries $\sigma(\pi_r, u)$ for \emph{all} $u \in V_r \setminus \{\pi_r\}$, $\access$ queries only $\lceil f(n_r) \rceil \le n_r$ other nodes $u$ (lines~\ref{line:queries1}--\ref{line:queries2}), where $n_r = |V_r|-1$.
Thus, while \kc\ always finds all positive labels involving the pivot $\pi_r$, \access\ can find them or not, with a probability that depends on $f$.
The function $f$ is called \emph{query rate function} and dictates the tradeoff between the clustering cost $\Delta$ and the number of queries $Q$, as we prove below.
Now, if any of the aforementioned $\lceil f(n_r) \rceil$ queries returns a positive label (line~\ref{line:sing}), then all the labels between $\pi_r$ and the remaining $u \in V_r$ are queried and the algorithm operates as \kc\ until the end of the recursive call; otherwise, the pivot becomes a singleton cluster which is removed from the set of nodes.
Another important difference is that $\access$ deterministically stops after at most $\lceil f(n) \rceil$ recursive calls (line~\ref{line:stop1}), declaring all remaining points as singleton clusters. The intuition is that with good probability the clusters not found within $\lceil f(n) \rceil$ rounds are small enough to be safely disregarded.
Since the choice of $f$ is delicate, we avoid trivialities by assuming $f$ is positive and smooth enough.
Formally:
\begin{definition}
\label{def:fn}
$f : \Nat\to\R$ is a \emph{query rate function} if $f(1)=1$, and $f(n) \le f(n+1) \le \big(1+\frac{1}{n}\big)f(n)$ for all $n \in \Nat$.
This implies $\frac{f(n+k)}{n+k} \le \frac{f(n)}{n}$ for all $k \ge 1$.
\end{definition}
We can now state formally our bounds for \access.
\begin{theorem}
\label{thm:access_cost}
For any query rate function $f$ and any labeling $\sigma$ on $n$ nodes, the expected cost $\E[\Delta_A]$ of the clustering output by $\access$ satisfies
\[
    \E[\Delta_A] \le 3\OPT + \frac{2e-1}{2(e-1)}\frac{n^2}{f(n)} + \frac{n}{e}~.
\]
The number of queries made by $\access$ is deterministically bounded as $Q \le n \lceil f(n) \rceil$. In the special case $f(n)=n$ for all $n\in\Nat$, $\access$ reduces to \kc\ and achieves $\E[\Delta_A] \le 3\OPT$ with $Q \le n^2$.
\end{theorem}
Note that Theorem~\ref{thm:access_cost} gives an upper bound on the error achievable when using $Q$ queries: since $Q=n f(n)$, the expected error is at most $3\OPT + \scO(n^3/Q)$. Furthermore, as one expects, if the learner is allowed to ask for all edge signs, then the {\em exact} bound of \kc\ is recovered (note that the first formula in Theorem~\ref{thm:access_cost} clearly does not take into account the special case when $f(n)=n$, which is considered in the last part of the statement).

\paragraph{Proof sketch.}
Look at a generic round $r$, and consider a pair of points $\{u,w\} \in V_r$.
The essence is that $\access$ can misclassify $\{u,w\}$ in one of two ways.
First, if $\sigma(u,w)=-1$, \access\ can choose as pivot $\pi_r$ a node $v$ such that $\sigma(v,u)=\sigma(v,w)=+1$.
In this case, if the condition on line~\ref{line:sing} holds, then \access\ will cluster $v$ together with $u$ and $w$, thus mistaking $\{u,w\}$.
If instead $\sigma(u,w)=+1$, then \access\ could mistake $\{u,w\}$ by pivoting on a node $v$ such that $\sigma(v,u)=+1$ and $\sigma(v,w)=-1$, and clustering together only $v$ and $u$.
Crucially, both cases imply the existence of a bad triangle $T=\{u,w,v\}$.
We charge each such mistake to exactly one bad triangle $T$, so that no triangle is charged twice.
The expected number of mistakes can then be bound by $3\OPT$ using the packing argument of~\cite{Ailon2008} for \kc.
Second, if $\sigma(u,w)=+1$ then \access\ could choose one of them, say $u$, as pivot $\pi_r$, and assign it to a singleton cluster.
This means the condition on line~\ref{line:sing} fails.
We can then bound the number of such mistakes as follows.
Suppose $\pi_r$ has ${cn}/{f(n)}$ positive labels towards $V_r$ for some $c \ge 0$.
Loosely speaking, we show that the check of line~\ref{line:sing} fails with probability $e^{-c}$, in which case ${cn}/{f(n)}$ mistakes are added.
In expectation, this gives ${cn e^{-c}}/{f(n)} = \scO\big({n}/{f(n)}\big)$ mistakes.
Over all $f(n) \le n$ rounds, this gives an overall $\scO\big({n^2}/{f(n)}\big)$.
(The actual proof has to take into account that all the quantities involved here are not constants, but random variables).

\subsection{\access{} with Early Stopping Strategy}
We can refine our algorithm \access\ so that, in some cases, it takes advantage of the structure of the input to reduce significantly the expected number of queries.
To this end we see the input as a graph $G$ with edges corresponding to positive labels (see above).
Suppose then $G$ contains a sufficiently small number $\scO(n^2/f(n))$ of edges.
Since \access\ performs up to $\ceil{f(n)}$ rounds, it could make $Q=\Theta(f(n)^2)$ queries.
However, with just $\ceil{f(n)}$ queries one could \emph{detect} that $G$ contains $\scO(n^2/f(n))$ edges, and immediately return the trivial clustering formed by all singletons.
The expected error would obviously be at most $\OPT + \scO(n^2/f(n))$, i.e.\ the same of Theorem~\ref{thm:access_cost}.
More generally, at each round $r$ with $\ceil{f(n_r)}$ queries one can check if the residual graph contains at least $n^2/f(n)$ edges; if the test fails, declaring all nodes in $V_r$ as singletons gives expected additional error $\scO(n^2/f(n))$.
The resulting algorithm is a variant of \access\ that we call \aggress\ (\access\ with Early Stopping Strategy).
The pseudocode can be found in the supplementary material.

First, we show \aggress\ gives guarantees virtually identical to \access\ (only, with $Q$ in expectation).
Formally:
\begin{theorem}
\label{thm:aggr1}
For any query rate function $f$ and any labeling $\sigma$ on $n$ nodes, the expected cost $\E[\Delta_A]$ of the clustering output by $\aggress$ satisfies
\[
    \E[\Delta_A] \le 3\OPT + 2\frac{n^2}{f(n)} + \frac{n}{e}~.
\]
Moreover, the expected number of queries performed by \aggress\ is $\E[Q] \le n(\ceil{f(n)} + 4)$.
\end{theorem}

Theorem~\ref{thm:aggr1} reassures us that \aggress\ is no worse than \access{}.
In fact, if most edges of $G$ belong to relatively large clusters (namely, all but $O(n^2/f(n))$ edges), then we can show \aggress\ uses much fewer queries than \access\ (in a nutshell, \aggress\ quickly finds all large clusters and then quits).
The following theorem captures the essence.
For simplicity we assume $\OPT=0$, i.e.\ $G$ is a disjoint union of cliques.
\begin{theorem}
\label{thm:aggr2}
Suppose $\OPT=0$ so $G$ is a union of disjoint cliques.
Let $C_1,\ldots,C_{\ell}$ be the cliques of $G$ in nondecreasing order of size.
Let $i'$ be the smallest $i$ such that $\sum_{j=1}^i |E_{C_j}| = \Omega(n^2/f(n))$, and let $h(n) = |C_{i'}|$.
Then \aggress\ makes in expectation $\E[Q] = \scO\big(n^2 \lg(n) / h(n)\big)$ queries.
\end{theorem}
As an example, say $f(n)=\sqrt{n}$ and $G$ contains $n^{1/3}$ cliques of $n^{2/3}$ nodes each.
Then for \access\ Theorem~\ref{thm:access_cost} gives $Q \le n f(n) = \scO(n^{3/2})$, while for \aggress\ Theorem~\ref{thm:aggr2} gives $\E[Q] = \scO(n^{4/3}\lg(n))$.

\section{Cluster recovery}
In the previous section we gave bounds on $\E[\Delta]$, the expected \emph{total} cost of the clustering.
However, in applications such as community detection and alike, the primary objective is recovering accurately the latent clusters of the graph, the sets of nodes that are ``close'' to cliques.
This is usually referred to as \emph{cluster recovery}.
For this problem, an algorithm that outputs a good approximation $\hat{C}$ of every latent cluster $C$ is preferable to an algorithm that minimizes $\E[\Delta]$ globally.
In this section we show that \access\ natively outputs clusters that are close to the latent clusters in the graph, thus acting as a cluster recovery tool.
We also show that, for a certain type of latent clusters, one can amplify the accuracy of \access\ via independent executions and recover all clusters exactly with high probability.

To capture the notion of ``latent cluster'', we introduce the concept of $(1-\epsilon)$\emph{-knit} set. As usual, we view $V,\sigma$ as a graph $G = (V,E)$ with $e \in E$ iff $\sigma(e)=+1$.
Let $E_C$ be the edges in the subgraph induced by $C\subseteq V$ and $\cut(C,\bar{C})$ be the edges between $C$ and $\bar{C} = V \setminus C$.
\begin{definition}
A subset $C \subseteq V$ is $(1-\epsilon)$\emph{-knit} if $\big|E_C\big| \ge (1-\epsilon)\binom{|C|}{2}$ and $\big|\!\cut(C,\bar{C})\big| \le \epsilon \binom{|C|}{2}$.
\end{definition}
Suppose now we have a cluster $\hat{C}$ as ``estimate'' of $C$. We quantify the distance between $C$ and $\hat{C}$ as the cardinality of their symmetric difference,
$
    \big|\hat{C} \oplus C\big| = \big|\hat{C} \setminus C\big| + \big|C \setminus \hat{C}\big|
$.
The goal is to obtain, for each $(1-\epsilon)$-knit set $C$ in the graph, a cluster $\hat{C}$ with $|\hat{C} \oplus C| = \scO(\epsilon|C|)$ for some small $\epsilon$.
We prove \access\ does exactly this.
Clearly, we must accept that if $C$ is too small, i.e.\ $|C| = o(n/f(n))$, then \access\ will miss $C$ entirely.
But, for $|C| = \Omega(n/f(n))$, we can prove $\E[|\hat{C} \oplus C|] = \scO(\epsilon|C|)$.
We point out that the property of being $(1-\epsilon)$-knit is rather weak for an algorithm, like \access, that is completely oblivious to the global topology of the cluster --- all what \access\ tries to do is to blindly cluster together all the neighbors of the current pivot.
In fact, consider a set $C$ formed by two disjoint cliques of equal size.
This set would be close to $\nicefrac{1}{2}$-knit, and yet \access\ would never produce a single cluster $\hat{C}$ corresponding to $C$.
Things can only worsen if we consider also the edges in $\cut(C,\bar{C})$, which can lead \access\ to assign the nodes of $C$ to several different clusters when pivoting on $\bar{C}$.
Hence it is not obvious that a $(1-\epsilon)$-knit set $C$ can be efficiently recovered by \access.

Note that this task can be seen as an \emph{adversarial} cluster recovery problem.
Initially, we start with a disjoint union of cliques, so that $\OPT=0$.
Then, an adversary flips the signs of some of the edges of the graph.
The goal is to retrieve every original clique that has not been perturbed excessively.
Note that we put no restriction on how the adversary can flip edges; therefore, this adversarial setting subsumes constrained adversaries.
For example, it subsumes the high-probability regime of the stochastic block model~\cite{HollandSBM} where edges are flipped according to some distribution.

We can now state our main cluster recovery bound for \access.
\begin{theorem}
\label{thm:recover}
For every $C \subseteq V$ that is $(1-\epsilon)$-knit, \access{} outputs a cluster $\hat{C}$ such that
$\E\big[|C \oplus \hat{C}|\big] \le 3\epsilon |C| + \min\!\big\{\frac{2n}{f(n)}, \big(1 - \frac{f(n)}{n}\big)|C| \big\} + |C|e^{-|C|f(n)/5n}$.
\end{theorem}
The $\min$ in the bound captures two different regimes: when $f(n)$ is very close to $n$, then $\E\big[|C \oplus \hat{C}|\big] = \scO(\epsilon |C|)$ independently of the size of $C$, but when $f(n) \ll n$ we need $|C|=\Omega(n/f(n))$, i.e., $|C|$ must be large enough to be found by \access.
\newcommand{\CR}{\ensuremath{\mathrm{ACR}}}
\newcommand{\id}{\ensuremath{\mathrm{id}}}

\subsection{Exact cluster recovery via amplification}
For certain latent clusters, one can get recovery guarantees significantly stronger than the ones given natively by \access\ (see Theorem~\ref{thm:recover}).
We start by introducing \emph{strongly} $(1-\epsilon)$\emph{-knit} sets (also known as quasi-cliques).
Recall that $\scN_v$ is the neighbor set of $v$ in the graph $G$ induced by the positive labels.
\begin{definition}
A subset $C \subseteq V$ is \emph{strongly} $(1-\epsilon)$\emph{-knit} if, for every $v \in C$, we have $\scN_v \subseteq C$ and $|\scN_v| \ge (1-\epsilon)(|C|-1)$.
\end{definition}
We remark that \access\ alone does not give better guarantees on strongly $(1-\epsilon)$-knit subsets than on $(1-\epsilon)$-knit subsets.
Suppose for example that $|\scN_v| = (1-\epsilon)(|C|-1)$ for all $v \in C$.
Then $C$ is strongly $(1-\epsilon)$-knit, and yet when pivoting on any $v \in C$ \access\ will inevitably produce a cluster $\hat{C}$ with $|\hat{C} \oplus C| \ge \epsilon |C|$, since the pivot has edges to less than $(1-\epsilon)|C|$ other nodes of $C$.

To bypass this limitation, we run \access\ several times to amplify the probability that every node in $C$ is found.
Recall that $V = [n]$.
Then, we define the id of a cluster $\hat{C}$ as the smallest node of $\hat{C}$.
The min-tagging rule is the following: when forming $\hat{C}$, use its id to tag all of its nodes.
Therefore, if $u_{\hat{C}} = \min\{u \in \hat{C}\}$ is the id of $\hat{C}$, we will set $\id(v) = u_{\hat{C}}$ for every $v \in \hat{C}$.
Consider now the following algorithm, called $\CR$ (Amplified Cluster Recovery).
First, $\CR$ performs $K$ independent runs of \access\ on input $V$, using the min-tagging rule on each run.
In this way, for each $v \in V$ we obtain $K$ tags $\id_1(v),\ldots,\id_K(v)$, one for each run.
Thereafter, for each $v \in V$ we select the tag that $v$ has received most often, breaking ties arbitrarily.
Finally, nodes with the same tag are clustered together.
One can prove that, with high probability, this clustering contains all strongly $(1-\epsilon)$-knit sets.
In other words, \CR\ with high probability recovers all such latent clusters \emph{exactly}.
Formally, we prove:
\begin{theorem}
\label{thm:recover2}
Let $\epsilon\le \frac{1}{10}$ and fix $p > 0$.
If \CR\ is run with $K = 48\ln\frac{n}{p}$, then the following holds with probability at least $1-p$: for every strongly $(1-\epsilon)$-knit $C$ with $|C| > 10\frac{n}{f(n)}$, the algorithm outputs a cluster $\hat{C}$ such that $\hat{C}=C$.
\end{theorem}
It is not immediately clear that one can extend this result by relaxing the notion of strongly $(1-\epsilon)$-knit set so to allow for edges between $C$ and the rest of the graph.
We just notice that, in that case, every node $v \in C$ could have a neighbor $x_v \in V \setminus C$ that is smaller than every node of $C$.
In this case, when pivoting on $v$ \access\ would tag $v$ with $x$ rather than with $u_C$, disrupting \CR.

\section{A fully additive scheme}
%
\iffalse %%%%%%%%%%%%%%%%%%%%%%%%%%%%%%%%%%%%%%%%%%%%%%%%%%%%%%%%%%%%%%%%%%%%%%%%
\renewcommand{\algorithmicrequire}{\textbf{Parameters:}}
\begin{algorithm}[h!]
\caption{
\label{alg:erm}
Invoked as $\ermcc(Q)$ where $V$.}
\begin{algorithmic}[1]
\setcounter{ALG@line}{0}
\Require{
Number of allowed queries $Q$.
}
\State Sample a set $S = ((x_1, y_1),\ldots,(x_Q, y_Q))$ from $(U, \sigma)^Q$
\State Output $\underset{\cs \in \cs_n}\argmin \frac{1}{Q} \sum_{i=1}^Q \Ind{h_{\cs}(x_i) \neq y_i}$
\end{algorithmic}
\end{algorithm}
\fi %%%%%%%%%%%%%%%%%%%%%%%%%%%%%%%%%%%%%%%%%%%%%%%%%%%%%%%%%%%%%%%%%%%%%%%%%%%%%
In this section, we introduce a(n inefficient) fully additive approximation algorithm achieving cost $\OPT + n^2\ve$ in high probability using order of $\frac{n}{\ve^2}$ queries. When $\OPT=0$, $Q = \frac{n}{\ve}\ln\frac{1}{\ve}$ suffices. Our algorithm combines uniform sampling with empirical risk minimization and is analyzed using VC theory.

First, note that CC can be formulated as an agnostic binary classification problem with binary classifiers $h_{\cs} : \scE \to \spin$ associated with each clustering $\cs$ of $V$ (recall that $\scE$ denotes the set of all pairs $\{u,v\}$ of distinct elements $u,v\in V$), and we assume $h_{\cs}(u,v) = +1$ iff $u$ and $v$ belong to the same cluster of $\cs$. Let $\Hn$ be the set of all such $h_{\cs}$. The risk of a classifier $h_{\cs}$ with respect to the uniform distribution over $\scE$ is $\Pr(h_{\cs}(e) \neq \sigma(e))$ where $e$ is drawn u.a.r.\ from $\scE$. It is easy to see that the risk of any classifier $h_{\cs}$ is directly related to $\Delta_{\cs}$, $\Pr\big(h_{\cs}(e) \neq \sigma(e)\big) = {\Delta_{\cs}}\big/{\binom{n}{2}}$.
\iffalse
\begin{align*}
\Pr\big(h_{\cs}(e) \neq \sigma(e)\big)
=
    \frac{\left|\theset{e \in \scE}{ h_{\cs}(e) \neq \sigma(e)}\right|}{\binom{n}{2}}
=
    \frac{1}{\binom{n}{2}}\underset{e\in\scE}\sum{\Ind{h_{\cs}(e) \neq \sigma(e)}}
=
    \frac{\Delta_{\cs}}{\binom{n}{2}}~.
\end{align*}
\fi
Hence, in particular,
$
    \OPT = \binom{n}{2}\min_{h\in\Hn} \Pr\big(h(e) \neq \sigma(e)\big)
$.
Now, it is well known ---see, e.g., \cite[Theorem 6.8]{Shalev-Shwartz:2014:UML:2621980}--- that we can minimize the risk to within an additive term of $\ve$ using the following procedure: query $\scO\big(d/\ve^2\big)$ edges drawn u.a.r.\ from $\scE$, where $d$ is the VC dimension of $\Hn$, and find the clustering $\cs$ such that $h_{\cs}$ makes the fewest mistakes on the sample. If there is $h^*\in\Hn$ with zero risk, then $\scO\big((d/\ve)\ln(1/\ve)\big)$ random queries suffice. A trivial upper bound on the VC dimension of $\Hn$ is $\log_2|\Hn| = \scO\big(n\ln n)$. The next result gives the exact value.
\begin{theorem}
\label{thm:vc_dim}
The VC dimension of the class $\Hn$ of all partitions of $n$ elements is $n-1$.
\end{theorem}
\begin{proof} Let $d$ be the VC dimension of $\Hn$. We view an instance of CC as the complete graph $K_n$ with edges labelled by $\sigma$. Let $T$ be any spanning tree of $K_n$. For any labeling $\sigma$, we can find a clustering $\cs$ of $V$ such that $h_{\cs}$ perfectly classifies the edges of $T$: simply remove the edges with label $-1$ in $T$ and consider the clusters formed by the resulting connected components. Hence $d \ge n-1$ because any spanning tree has exactly $n-1$ edges.
On the other hand, any set of $n$ edges must contain at least a cycle. It is easy to see that no clustering $\cs$  makes $h_{\cs}$ consistent with the labeling $\sigma$ that gives positive labels to all edges in the cycle but one. Hence $d < n$.
\end{proof}
An immediate consequence of the above is the following.
\begin{theorem}
\label{thm:fa_scheme}
There exists a randomized algorithm $A$ that, for all $0 < \ve < 1$, finds a clustering $\cs$ satisfying $\Delta_{\cs}\le \OPT + \scO\big(n^2\epsilon\big)$ with high probability while using $Q = \scO\big(\frac{n}{\ve^2}\big)$ queries. Moreover, if $\OPT=0$, then $Q = \scO\big(\frac{n}{\ve}\ln\frac{1}{\ve}\big)$ queries are enough to find a clustering $\cs$ satisfying $\Delta_{\cs} = \scO\big(n^2\epsilon\big)$.
\end{theorem}

\section{Lower bounds}
In this section we give two lower bounds on the expected clustering error of any (possibly randomized) algorithm.
The first bound holds for $\OPT=0$, and applies to algorithms using a deterministically bounded number of queries.
This bound is based on a construction from \cite[Lemma 11]{cesa2015complexity} and related to kernel-based learning.
\begin{theorem}
\label{th:lower}
For any $\ve > 0$ such that $\frac{1}{\ve}$ is an even integer, and for every (possibly randomized) learning algorithm asking fewer than $\frac{1}{50\ve^2}$ queries with probability $1$, there exists a labeling $\sigma$ on $n \ge \frac{16}{\ve}\ln\frac{1}{\ve}$ nodes such that $\OPT=0$ and the expected cost of the algorithm is at least $\frac{n^2\ve}{8}$.
\end{theorem}

\ifproofs
\begin{proof}
We prove that there exists a distribution over labelings $\sigma$ with $\OPT=0$ on which any deterministic algorithm has expected cost at least $\frac{n\ve^2}{8}$. Yao's minimax principle then implies the claimed result.
 
Given $V = \{1,\ldots,n\}$, we define $\sigma$ by a random partition of the vertices in $d \ge 2$ isolated cliques $T_1,\ldots,T_d$ such that $\sigma(v,v') = +1$ if and only if $v$ and $v'$ belong to the same clique. The cliques are formed by assigning each node $v \in V$ to a clique $I_v$ drawn uniformly at random with replacement from $\{1,\dots,d\}$, so that $T_i = \theset{v \in V}{I_v=i}$. Consider a deterministic algorithm making queries $\{s_t,r_t\} \in \scE$. Let $E_i$ be the event that the algorithm never queries a pair of nodes in $T_i$ with $|T_i| \ge \frac{n}{2d} > 5$. Apply Lemma~\ref{lem:cluster} with $d = \frac{1}{\ve}$. This implies that the expected number of non-queried clusters of size at least $\frac{n}{2d}$ is at least $\frac{d}{2} = \frac{1}{2\ve}$. The overall expected cost of ignoring these clusters is therefore at least
\[
	\frac{d}{2}\left(\frac{n}{2d}\right)^2 = \frac{n^2}{8d} = \frac{\ve n^2}{8}
\]
and this concludes the proof.
\end{proof}
\begin{lemma}
\label{lem:cluster}
Suppose $d > 0$ is even, $n \ge 16d\ln d$, and $B < \frac{d^2}{50}$. Then for any deterministic learning algorithm making at most $B$ queries,
\[
    \sum_{i=1}^d \Pr(E_i) > \frac{d}{2}~.
\]
\end{lemma}
\begin{proof}
For each query $\{s_t,r_t\}$ we define the set $L_t$ of all cliques $T_i$ such that $s_t\not\in T_i$ and some edge containing both $s_t$ and a node of $T_i$ was previously queried. The set $R_t$ is defined similarly using $r_t$. Formally,
\begin{align*}
    L_t = & \theset{ i }{ (\exists \tau < t) \; s_\tau=s_t \,\wedge\, r_\tau\in T_i \,\wedge\, \sigma(s_{\tau},r_{\tau}) = -1 }
\\
    R_t = & \theset{ i }{ (\exists \tau < t) \; r_\tau=r_t \,\wedge\, s_\tau\in T_i \,\wedge\, \sigma(s_{\tau},r_{\tau}) = -1- }~.
\end{align*}
Let $D_t$ be the event that the $t$-th query discovers a new clique of size at least $\frac{n}{2d}$, and let $P_t = \max\bigl\{|L_t|,|R_t|\bigr\}$. Using this notation,
\begin{align}
\label{eq:clubound}
    \sum_{t=1}^B \Ind{D_t}
=
    \sum_{t=1}^B \Ind{D_t \,\wedge\, P_t < d/2} + \underbrace{\sum_{t=1}^B \Ind{D_t \,\wedge\, P_t \ge d/2}}_{N}~.
\end{align}
We will now show that unless $B \ge \tfrac{d^2}{50}$, we can upper bound $N$ deterministically by $\sqrt{2B}$.

Suppose $N > \tfrac{d}{2}$, and let $t_1,\dots,t_N$ be the times $t_k$ such that $\Ind{D_{t_k} \,\wedge\, P_{t_k} \ge d/2} = 1$. Now fix some $k$ and note that, because the clique to which $s_{t_k}$ and $r_{t_k}$ both belong is discovered, neither $s_{t_k}$ nor $r_{t_k}$ can occur in a future query $\{s_t,r_t\})$ that discovers a new clique. Therefore, in order to have $\Ind{D_t \,\wedge\, P_t \ge d/2} = 1$ for $N > \tfrac{d}{2}$ times, at least
\[
    \binom{N}{2} \ge \frac{d^2}{8}
\]
queries must be made, since each one of the other $N-1 \ge \frac{d}{2}$ discovered cliques can contribute with at most a query to making $P_t \ge \tfrac{d}{2}$. So, it takes at least
$
    B \ge \frac{d^2}{8}
$
queries to discover the first $\frac{d}{2}$ cliques of size at least two, which contradicts the lemma's assumption that $B \le \tfrac{d^2}{16}$. Therefore, $N \le \tfrac{d}{2}$.

Using the same logic as before, in order to have $\Ind{D_t \,\wedge\, P_t \ge d/2} = 1$ for $N \le \tfrac{d}{2}$ times, at least
\[
    \frac{d}{2} + \left(\frac{d}{2} - 1\right) + \dots + \left(\frac{d}{2} - N + 1\right)
\]
queries must be made. So, it must be
\[
    B \ge \sum_{k=1}^N \left(\frac{d}{2} - (k - 1)\right) = (d+1)\frac{N}{2} - \frac{N^2}{2}
\]
or, equivalently, $N^2 -(d+1)N + 2B \ge 0$. Solving this quadratic inequality for $N$, and using the hypothesis $N\le \tfrac{d}{2}$, we have that $N\leq \frac{(d+1)-\sqrt{(d+1)^2-8B}}{2}$. Using the assumption that $B\leq \tfrac{d^2}{50}$ we get that $N\leq \sqrt{2B}$.

We now bound the first term of~(\ref{eq:clubound}) in expectation. The event $D_t$ is equivalent to $s_t,r_t \in T_i$ for some $i \in \neg L_t \cap \neg R_t$, where for any $S \subseteq \{1,\dots,d\}$ we use $\neg S$ to denote $\{1,\dots,d\}\setminus S$.

Let $\Pr_t = \Pr\bigl(\,\cdot\mid P_t < d/2 \bigr)$. For $L',R'$ ranging over all subsets of $\{1,\dots,d\}$ of size strictly less than $\tfrac{d}{2}$,
\begin{align}
\nonumber
    \Pr_t(D_t)
&=
    \sum_{L',R'}\sum_{i \in \neg L' \cap \neg R'} \!\Pr_t\bigl(s_t \in T_i \,\wedge\, r_t \in T_i \,\big|\, L_t = L',\, R_t = R'\bigr)\,\Pr_t(L_t = L' \,\wedge\, R_t = R')
\\ &=
\label{eq:indep}
    \sum_{L',R'}\sum_{i \in \neg L' \cap \neg R'} \!\Pr_t\bigl(s_t \in T_i \,\big|\, L_t = L'\bigr)\,\Pr_t\bigl(r_t \in T_i \,\big|\, R_t = R'\bigr)\,\Pr_t(L_t = L' \,\wedge\, R_t = R')
\\ &=
\label{eq:split}
    \sum_{L',R'}\sum_{i \in \neg L' \cap \neg R'} \frac{1}{|\neg L'|}\,\frac{1}{|\neg R'|}\,\Pr_t(L_t = L' \,\wedge\, R_t = R')
\\ &=
\nonumber
    \sum_{L',R'} \frac{|\neg L' \cap \neg R'|}{|\neg L'|\,|\neg R'|}\,\Pr_t(L_t = L' \,\wedge\, R_t = R')
\\ &\le
\label{eq:count}
	\frac{2}{d}~.
\end{align}
Equality~(\ref{eq:indep}) holds because $P_t = \max\{L_t,R_t\} < \frac{d}{2}$ implies that there are at least two remaining cliques to which $s_t$ and $r_t$ could belong, and each node is independently assigned to one of these cliques. Equality~(\ref{eq:split}) holds because, by definition of $L_t$, the clique of $s_t$ is not in $L_t$, and there were no previous queries involving $s_t$ and a node belonging to a clique in $\neg L_t$ (similarly for $r_t$). Finally, (\ref{eq:count}) holds because $|\neg L'| \ge \tfrac{d}{2}$, $|\neg R'| \ge \tfrac{d}{2}$, and $|\neg L' \cap \neg R'| \le \min\{|\neg L'|,|\neg R'|\}$. Therefore,
\begin{align*}
    \sum_{t=1}^B \Pr\bigl(D_t \,\wedge\, P_t < d/2\bigr)
\le
    \sum_{t=1}^B \Pr\bigl(D_t \mid P_t < d/2\bigr)
\le
    \frac{2B}{d}~.
\end{align*}
Putting everything together,
\begin{equation}\label{eq:clubound1}
\E\left[\sum_{t=1}^B \Ind{D_t}\right] \leq \frac{2B}{d} + \sqrt{2B}~.
\end{equation}
On the other hand, we have
\begin{equation}
\label{eq:clubound2}
	\sum_{t=1}^B \Ind{D_t}
=
	\sum_{i=1}^{d}\Big(\Ind{|T_i|\ge\tfrac{n}{2d}}-\Ind{E_i}\Big)
=
	d - \sum_{i=1}^{d}\Big(\Ind{|T_i|<\tfrac{n}{2d}}+\Ind{E_i}\Big)
\end{equation}
Combining \eqref{eq:clubound1} and \eqref{eq:clubound2}, we get that
\[
	\sum_{i=1}^{d}\Pr(E_i)
\ge
	d-\sum_{i=1}^{d}\Pr\big(|T_i|<\tfrac{n}{2d}\big) - \frac{2B}{d}-\sqrt{2B}~.
\]
By Chernoff-Hoeffding bound,
$
	\Pr\big(|T_i|<\tfrac{n}{2d}\big) \le \frac{1}{d^2}
$
for each $i=1,\dots,d$ when $n \ge 16d\ln d$. Therefore,
\[
	\sum_{i=1}^{d}\Pr(E_i)
\ge
	d - \frac{2B+1}{d} - \sqrt{2B}~.
\]
To finish the proof, suppose on the contrary that $\sum_{i=1}^{d}\Pr(E_i)\leq \frac{d}{2}$. Then from the inequality above, we
would get that
\[
	\frac{d}{2}
\ge
	d - \frac{2B+1}{d} - \sqrt{2B}
\]
which implies $B\geq \left(\frac{2-\sqrt{2}}{4}\right)^2 d^2> \frac{d^2}{50}$, contradicting the assumptions. Therefore, we must have $\sum_{i=1}^{d}\Pr(E_i)> \frac{d}{2}$ as required.
\end{proof}
\fi

Our second bound relaxed the assumption on $\OPT$.
It uses essentially the same construction of~\cite[Lemma~6.1]{bonchi2013local}, giving asymptotically the same guarantees.
However, the bound of \cite{bonchi2013local} applies only to a very restricted class of algorithms: namely, those where the number $q_v$ of queries involving any specific node $v \in V$ is deterministically bounded.
This rules out a vast class of algorithms, including \kc, \access, and \aggress, where the number of queries involving a node is a function of the random choices of the algorithm.
Our lower bound is instead fully general: it holds unconditionally for \emph{any} randomized algorithm, with no restriction on what or how many pairs of points are queried.
\begin{theorem}
\label{thm:new_LB}
Choose any function $\epsilon=\epsilon(n)$ such that $\Omega\big(\frac{1}{n}\big) \le \epsilon \le \frac{1}{2}$ and $\frac{1}{\epsilon} \in \mathbb{N}$.
For every (possibly randomized) learning algorithm and any $n_0>0$ there exists a labeling $\sigma$ on $n\ge n_0$ nodes such that the algorithm has expected error $\E[\Delta] \ge  \OPT + \frac{n^2\epsilon}{80}$ whenever its expected number of queries satisfies $\E[Q] < \frac{n}{80 \,\epsilon}$.
\end{theorem}
In fact, the bound of Theorem~\ref{thm:new_LB} can be put in a more general form: for any constant $c \ge 1$, the expected error is at least $c \cdot \OPT + A(c)$ where $A(c) = \Omega(n^2 \epsilon)$ is an additive term with constant factors depending on $c$ (see the proof).
Thus, our algorithms \access\ and \aggress\ are essentially optimal in the sense that, for $c=3$, they guarantee an optimal additive error up to constant factors.

\section{Experiments}
We verify experimentally the tradeoff between clustering cost and number of queries of \access, using six datasets from \cite{mazumdar2017clustering,NIPS2017_7054}.
Four datasets come from real-world data, and two are synthetic; all of them provide a ground-truth partitioning of some set $V$ of nodes.
Here we show results for one real-world dataset (\texttt{cora}, with $|V|$=1879 and 191 clusters) and one synthetic dataset (\texttt{skew}, with $|V|$=900 and 30 clusters).
Results for the remaining datasets are similar and can be found in the supplementary material.
Since the original datasets have $\OPT=0$, we derived perturbed versions where $\OPT > 0$ as follows.
First, for each $\eta \in\{0,0.1,0.5,1\}$ we let $p = \eta |E|/\binom{n}{2}$ where $|E|$ is the number of edges (positive labels) in the dataset (so $\eta$ is the expected number of flipped edges measured as a multiple of $|E|$).
Then, we flipped the label of each pair of nodes independently with probability $p$.
Obviously for $p=0$ we have the original dataset.

For every dataset and its perturbed versions we then proceeded as follows.
For $\alpha=0,0.05,...,0.95,1$, we set the query rate function to  $f(x)=x^{\alpha}$.
Then we ran $20$ independent executions of \access, and computed the average number of queries $\mu_Q$ and average clustering cost $\mu_{\Delta}$.
The variance was often negligible, but is reported in the full plots in the supplementary material.
The tradeoff between $\mu_{\Delta}$ and $\mu_Q$ is depicted in Figure~\ref{f:plot}, where the circular marker highlights the case $f(x)=x$, i.e.\ \kc.

\iffalse
\begin{table}[h]
\caption{
\label{t:datasets}
Description of the datasets.
}
\centering
\begin{tabular}{lccc}
\toprule
Datasets & Type & $|V|$ & \#Clusters \\
\midrule
%landmarks & Real & 266 & 12 \\
%captchas & Real & 244 & 69 \\
%gym & Real & 94 & 12 \\
skew & Synthetic & 900 & 30 \\
%sqrt & Synthetic & 900 & 30 \\
cora & Real-world & 1879 & 191 \\
\midrule
\end{tabular}
\end{table}
\fi

\begin{figure}[h]
\centering
\begin{subfigure}[b]{0.46\linewidth}
\includegraphics[width=\linewidth]{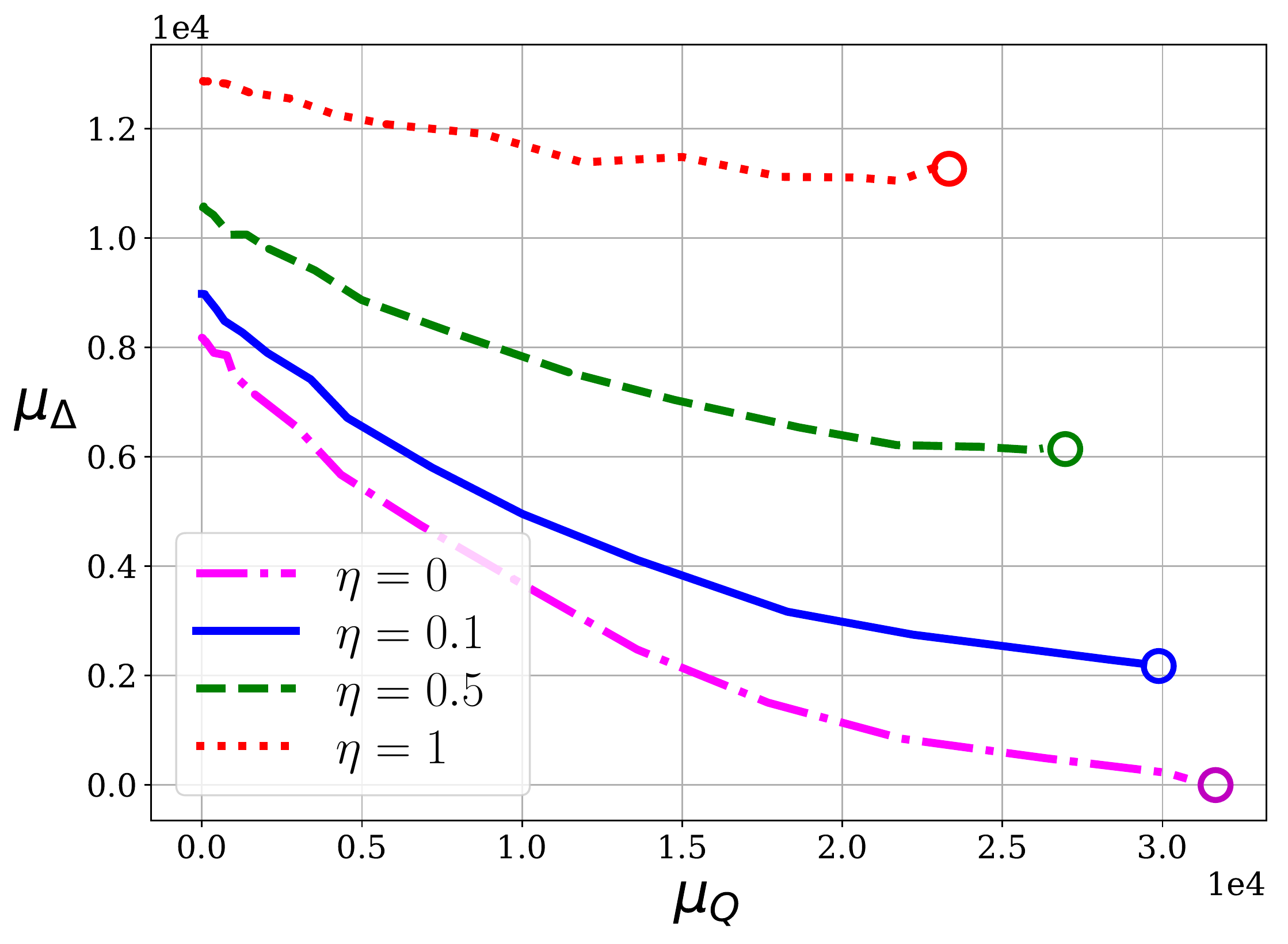}
\caption{skew.}
\end{subfigure}
\hspace*{20pt}
\begin{subfigure}[b]{0.46\linewidth}
\includegraphics[width=\linewidth]{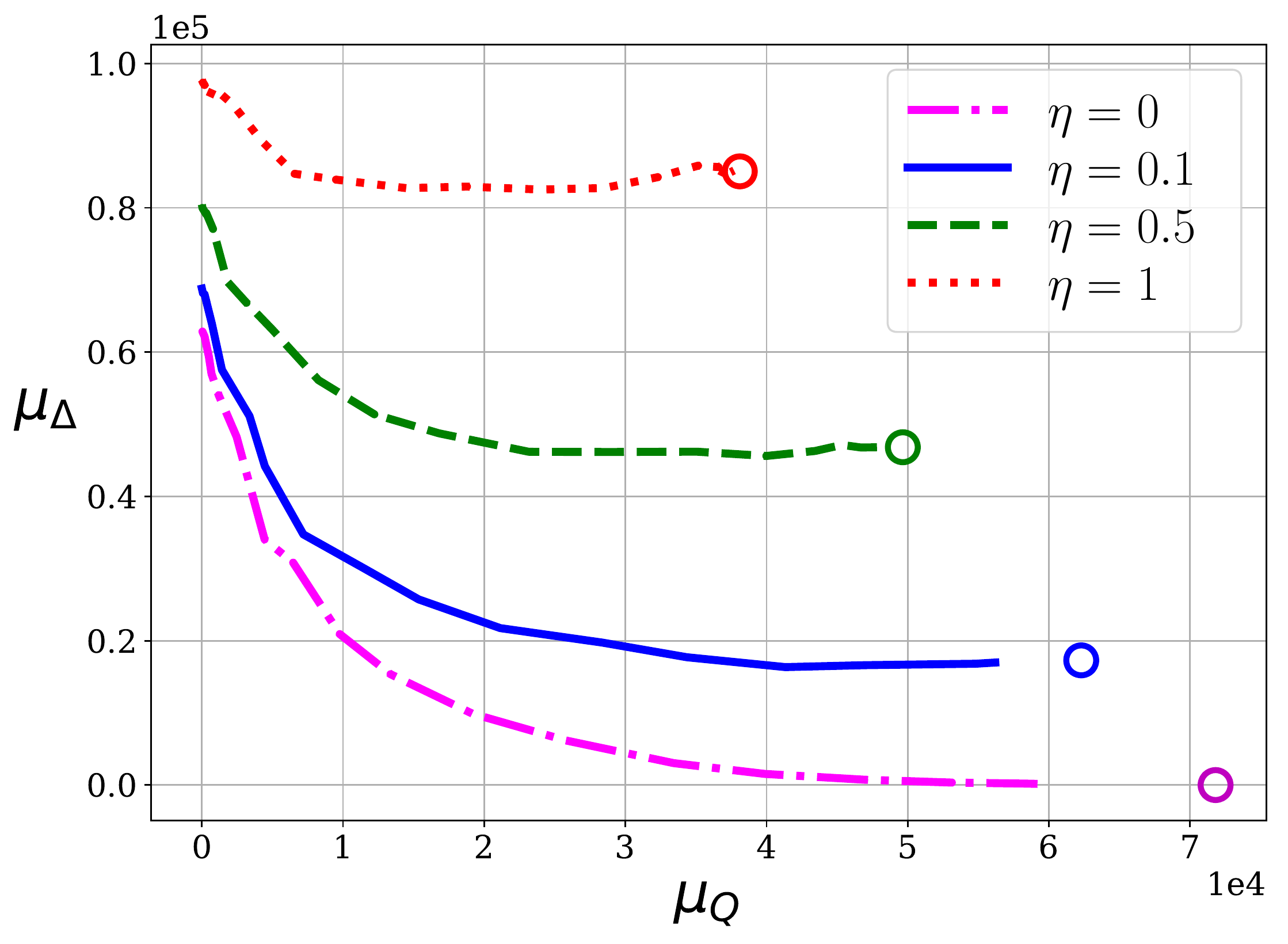} 
\caption{cora.}
\end{subfigure}
\caption{Performance of \access. \label{f:plot}}
\end{figure}

The clustering cost clearly drops as the number of queries increases.
This drop is particularly marked on \texttt{cora}, where \access\ achieves a clustering cost close to that of \kc\ using an order of magnitude fewer queries.
%For large $p$, surprisingly, $\access$ with $f(x)=1$ outperforms \kc; that is, spending more queries worsens the clustering quality.
%A plausible explanation is that a low query budget makes \access\ insensitive to spurious edges that can ``trick'' \kc\ into breaking apart large ground-truth communities.
It is also worth noting that, for the case $\OPT=0$, the measured clustering cost achieved by \access\ is $2$ to $3$ times lower than the theoretical bound of $\approx 3.8 n^3 / Q$  given by Theorem~\ref{thm:access_cost}.

\subsubsection*{Acknowledgements}

The authors gratefully acknowledge partial support by the Google Focused Award \enquote{Algorithms and Learning for AI} (ALL4AI).
Marco Bressan and Fabio Vitale are also supported in part by the ERC Starting Grant DMAP 680153 and by the ``Dipartimenti di Eccellenza 2018-2022'' grant awarded to the Department of Computer Science of the Sapienza University of Rome.
Nicolò Cesa-Bianchi is also supported by the MIUR PRIN grant \textsl{Algorithms, Games, and Digital Markets} (ALGADIMAR).

\bibliographystyle{plainnat}
\bibliography{biblio}

\clearpage
\appendix

\newcommand{\et}{\,\wedge\,}

\section*{APPENDIX}

\section{Probability bounds}
\label{apx:chernoff_bounds}
We give Chernoff-type probability bounds that can be found in e.g.~\cite{Dubhashi2009} and that we repeatedly use in our proofs.
Let $X_1,\ldots,X_n$ be binary random variables. We say that $X_1,\ldots,X_n$ are non-positively correlated if for all $I \subseteq \{1,\ldots,n\}$ we have:
\begin{align}
\prob[\forall i \in I: X_i=0] \leq \prod_{i \in I} \prob[X_i=0] \quad \text{and} \quad
\prob[\forall i \in I: X_i=1] \leq \prod_{i \in I} \prob[X_i=1]
\end{align}
The following holds:
\begin{lemma}
\label{lem:chernoff}
Let $X_1,\ldots,X_n$ be independent or, more generally, non-positively correlated binary random variables. Let $a_1,\ldots,a_n \in [0,1]$ and $X=\sum_{i=1}^{n}a_iX_i$. Then, for any $\delta > 0$, we have:
\begin{align}
\prob[X < (1-\delta)\E[X]] &< e^{-\frac{\delta^2}{2}\E[X]} \\
\prob[X > (1+\delta)\E[X]] &< e^{-\frac{\delta^2}{2+\delta}\E[X]} 
\end{align}
\end{lemma}
\noindent

\section{Supplementary Material for Section 3}
\label{s:access}

\subsection{Pseudocode of \access}
For ease of reference we report the pseudocode of \access\ below.
\setcounter{algorithm}{0}

\subsection{Proof of Theorem 1}
We refer to the pseudocode above (Algorithm~\ref{alg:access}).
We use $V_r$ to denote the set of remaining nodes at the beginning of the $r$-th recursive call, and we let $n_r = |V_r|-1$.
Hence $V_1 = V$ and $n_1=n-1$.
If the condition in the \textbf{if} statement on line~\ref{line:sing} is not true, then $C_r$ is a singleton cluster.
We denote by $\Vsing$ the set nodes that are output as singleton clusters.

Let $\Gamma_A$ be the set of mistaken edges for the clustering output by $\access$ and let $\Delta_A = \big|\Gamma_A\big|$ be the cost of this clustering. Note that, in any recursive call, $\access$ misclassifies an edge $e=\{u,w\}$ if and only if $e$ is part of a bad triangle whose third node $v$ is chosen as pivot and does not become a singleton cluster, or if $\sigma(e)=+1$ and at least one of $u,w$ becomes a singleton cluster. More formally, $\access$ misclassifies an edge $e = \{u,w\}$ if and only if one of the following three disjoint events holds:
\begin{enumerate}[topsep=0pt,parsep=0pt,itemsep=0pt]
\item[$B_1(e)$:] \label{item:ev1} There exists $r \le \ceil{f(n-1)}$ and a bad triangle $T \equiv \{u,v,w\} \subseteq V_r$ such that $\pi_r=v$ and $v\not\in\Vsing$.
\item[$B_2(e)$:] \label{item:ev2} There exists $r \le \ceil{f(n-1)}$ such that $u,w \in V_r$ with $\sigma(u,w)=+1$ and $\pi_r \in \{u,w\} \cap \Vsing$.
\item[$B_3(e)$:] \label{item:ev3} \access\ stops after $\ceil{f(n-1)}$ rounds without removing neither $u$ nor $w$, and $\sigma(u,w)=+1$.
\end{enumerate}
Therefore the indicator variable for the event ``$e$ is mistaken'' is: 
\begin{align*}
    \Ind{e \in \Gamma_A}
&=
%    \Ind{(\exists\,r) (\exists\,T\in\scT) \,:\, T \subseteq V_r \,\wedge\, e \subset T \,\wedge\, \pi_r\in T\setminus e}
\Ind{B_1(e)} +
    \Ind{B_2(e)} + \Ind{B_3(e)}
\end{align*}
The expected cost of the clustering is therefore:
\begin{align}
    \E[\Delta_A] = \sum_{e \in \scE} \Pr(B_1(e)) + \sum_{e \in \scE} \Pr(B_2(e)) + \sum_{e \in \scE} \Pr(B_3(e))
\end{align}
We proceed to bound the three terms separately.
\paragraph{Bounding $\sum_{e \in \scE} \Pr(B_1(e))$.}
Fix an arbitrary edge $e=\{u,w\}$.
Note that, if $B_1(e)$ occurs, then $T$ is unique, i.e.\ exactly one bad triangle $T$ in $V$ satisfies the definition of $B_1(e)$.
Each occurrence of $B_1(e)$ can thus be charged to a single bad triangle $T$. We may thus write
\begin{align*}
    \sum_{e \in \scE} \Ind{B_1(e)} &=
    \sum_{e \in \scE}  \Ind{(\exists\,r) (\exists\,T\in\scT) \,:\, T \subseteq V_r \,\wedge\, e \subset T \,\wedge\, \pi_r\in T\setminus e  \,\wedge\, \pi_r\not\in\Vsing}
\\ &=
    \sum_{T \in \scT} \Ind{(\exists\,r)  \,:\, T \subseteq V_r \,\wedge\, \pi_r\in T  \,\wedge\, \pi_r\not\in\Vsing}
\\ &\le
    \sum_{T \in \scT} \Ind{A_T} 
\end{align*}
where $A_T \equiv \big\{(\exists\,r) \,:\, T \subseteq V_r \,\wedge\, \pi_r \in T\big\}$.
Let us then bound $\sum_{T \in \scT}\Pr(A_T)$.
Let $\scT(e) \equiv \theset{T'\in\scT}{e\in T'}$. We use the following fact extracted from the proof of \citep[Theorem~6.1]{Ailon2008}. If $\theset{\beta_T \ge 0}{T\in\scT}$ is a set of weights on the bad triangles such that $\sum_{T \in \scT(e)} \beta_T \le 1$ for all $e \in \scE$, then $\sum_{T \in \scT} \beta_T \le \OPT$. Given $e \in \scE$ and $T \in \scT$, let $F_T(e)$ be the event corresponding to $T$ being the first triangle in the set $\scT(e)$ such that $T \in V_r$ and $\pi_r \in T\setminus e$ for some $r$. Now if $F_T(e)$ holds then $A_T$ holds and no other $A_{T'}$ for $T'\in \scT(e)\setminus\{T\}$ holds. Therefore
\[
     \sum_{T \in \scT(e)} \Ind{A_T \,\wedge\, F_T(e)} = 1~.
\]
If $A_T$ holds for some $r_0$, then it cannot hold for any other $r > r_0$ because $\pi_{r_0} \in T$ implies that for all $r > r_0$ we have $\pi_{r_0} \not\in V_r$ implying $T \not\subseteq V_r$. Hence, given that $A_T$ holds for $r_0$, if $F_T(e)$ holds too, then it holds for the same $r_0$ by construction. This implies that $\Pr\big (F_T(e) \mid A_T \big) = \frac{1}{3}$ because $\access$ chooses the pivot u.a.r.\ from the nodes in $V_{r_0}$. Thus, for each $e \in E$ we can write
\begin{equation}
    1 = \sum_{T \in \scT(e)} \Pr\big(A_T \,\wedge\, F_T(e)\big) = \sum_{T \in \scT(e)} \Pr\big(F_T(e) \mid A_T\big) \Pr(A_T) = \sum_{T \in \scT(e)} \frac{1}{3}\Pr(A_T)~.
\label{eq:fp}
\end{equation}
Choosing $\beta_T = \frac{1}{3}\Pr(A_T)$ we get $\sum_{T \in \scT} \Pr(A_T) \le 3\OPT$.

In the proof of \kc, the condition $\sum_{T \in \scT(e)} \beta_T \le 1$ was ensured by considering events $G_T(e) = A_T \,\wedge\, e \in \Gamma_A$. Indeed, in \kc\ the events $\theset{G_T(e)}{T\in\scT(e)}$ are disjoint, because $G_T(e)$ holds iff $T$ is the first and only triangle in $\scT(e)$ whose node opposite to $e$ is chosen as pivot. For $\access$ this is not true because a pivot can become a singleton cluster, which does not cause $e \in \Gamma_A$ necessarily to hold.

\paragraph{Bounding $\sum_{e \in \scE}\prob(B_2(e))$.} For any $u \in V_r$, let $d^+_r(u) = \big|\theset{v \in V_r}{\sigma(u,v) = +1}\big|$. We have:
\[
    \sum_{e \in \scE} \Ind{B_2(e)} = \frac{1}{2} \sum_{u \in V} \sum_{r=1}^{\ceil{f(n-1)}} \Ind{\pi_r = u \,\wedge\, \pi_r\in\Vsing} d^+_r(u)~.
\]
Taking expectations with respect to the randomization of $\access$,
\begin{align*}
    \sum_{e \in \scE} \Pr\big(B_2(e)\big)
&=
   \frac{1}{2} \sum_{u \in V} \sum_{r=1}^{\ceil{f(n-1)}} \E\Big[\Ind{\pi_r = u \,\wedge\, \pi_r\in\Vsing} d^+_r(u) \Big]
\\ &=
    \frac{1}{2} \sum_{u \in V} \sum_{r=1}^{\ceil{f(n-1)}} \E\Big[\Ind{\pi_r\in\Vsing} d^+_r(u) \,\Big|\, \pi_r = u \Big] \Pr(\pi_r = u)
\end{align*}
For any round $r$, let $H_{r-1}$ be the sequence of random draws made by the algorithm before round $r$.
Then $\Pr\big(\pi_r\in\Vsing \,\big|\, \pi_r = u,\, H_{r-1} \big)d^+_r(u) = 0$ if either $d^+_r(u) = 0$, or $d^+_r(u) \ge 1$ and $d_r^-(u) < \ceil{f(n_r)}$. Otherwise,
\begin{equation}
\label{eq:hyper}
    \Pr\big(\pi_r\in\Vsing \,\big|\, \pi_r = u,\, H_{r-1} \big)
=
    \!\!\!\prod_{j=0}^{\ceil{f(n_r)}-1} \frac{d_r^-(u)-j}{n_r-j}
\le
    \left(\frac{d_r^-(u)}{n_r}\right)^{\ceil{f(n_r)}}
\!\!\!\!\!=
     \left(1 - \frac{d^+_r(u)}{n_r}\right)^{\ceil{f(n_r)}}
\end{equation}
where the inequality holds because $d_r^-(u) \le n_r$. Therefore, when $d^+_r(u) \ge 1$ and $d_r^-(u) \ge \ceil{f(n_r)}$,
\begin{align*}
    \E\Big[\Ind{\pi_r\in\Vsing} d^+_r(u) \,\Big|\, \pi_r = u,\, H_{r-1} \Big]
&=
    \Pr\big(\pi_r\in\Vsing \,\big|\, \pi_r = u,\, H_{r-1} \big) d^+_r(u)
\\ &=
    \left(1 - \frac{d^+_r(u)}{n_r}\right)^{\ceil{f(n_r)}}d^+_r(u)
\\ &=
    \left(1 - \frac{d^+_r(u)}{n_r}\right)^{\ceil{f(n_r)}}d^+_r(u)
\\ &\le
    \exp\left(-\frac{d^+_r(u)\ceil{f(n_r)}}{n_r}\right)d^+_r(u)
\\ &\le
    \max_{z > 0} \exp\left(-\frac{z\,\ceil{f(n_r)}}{n_r}\right) z
\\ &\le
    \frac{n_r}{e \ceil{f(n_r)}}
\\ &\le
    \frac{n_r}{e f(n_r)}~.
\end{align*}
Combining with the above, this implies
\begin{align*}
    \sum_{e \in \scE} \Pr\big(B_2(e)\big)
\le
     \frac{1}{2e} \sum_{r=1}^{\ceil{f(n-1)}} \E\left[\frac{n_r}{f(n_r)} \right] \le
     \frac{1}{2e} \sum_{r=1}^{\ceil{f(n-1)}} \frac{n}{f(n)}
\le
     \frac{n}{e}
\end{align*}
where we used the facts that $n_r \le n$ and the properties of $f$.

\paragraph{Bounding $\sum_{e \in \scE}\prob(B_3(e))$.}
Let $\Vfin$ be the remaining vertices in $V_r$ after the algorithm stops and assume $|\Vfin| > 1$ (so that there is at least a query left).
Let $\nfin = |\Vfin|-1$ and, for any $u \in \Vfin$, let $\dfin^+(u) = \big|\theset{v \in \Vfin}{\sigma(u,v) = +1}\big|$. In what follows, we conventionally assume $V_{r} \equiv \Vfin$ for any $r > \ceil{f(n-1)}$, and similarly for $\nfin$ and $\dfin^+$. We have
\begin{align*}
    \sum_{e \in \scE} \Ind{B_3(e)}
=
    \frac{1}{2} \sum_{u \in \Vfin} \dfin^+(u)
 \le
    \frac{1}{2} \left( \sum_{u \in \Vfin} \frac{\nfin}{\ceil{f(\nfin)}} + \sum_{u \in \Vfin} \Ind{\dfin^+(u) > \frac{\nfin}{\ceil{f(\nfin)}}} \dfin^+(u)\right).
\end{align*}
Fix some $r \le \ceil{f(n-1)}$.
Given any vertex $v \in V_r$ with $d^+_r(v) \ge \frac{n_r}{\ceil{f(n_r)}}$, let $E_r(v)$ be the event that, at round $r$, $\access$ queries $\sigma(v,u)$ for all $u \in V_r\setminus\{v\}$.
Introduce the notation $S_r = \sum_{u \in V_r} \Ind{d_r^+(u) > \frac{n_r}{\ceil{f(n_r)}}} d_r^+(u)$ with $S_r = \Sfin$ for all $r > \ceil{f(n)}$, and let $\delta_r = n_r - n_{r+1}$ be the number of nodes that are removed from $V_r$ at the end of the $r$-th recursive call.
Then
\[
    \delta_r
\ge
   \Ind{E_r(\pi_r)} d_r^+(\pi_r)
\ge
   \Ind{d_r^+(\pi_r) > \frac{n_r}{\ceil{f(n_r)}}} \Ind{E_r(\pi_r)} d_r^+(\pi_r)
\]
and
\[
    \E[\delta_r \mid H_{r-1}]
\ge
    \sum_{v \in V_r} \Ind{d_r^+(v) > \frac{n_r}{\ceil{f(n_r)}}} \Pr\big(E_r(v) \mid \pi_r = v,\, H_{r-1} \big) \Pr(\pi_r = v \mid H_{r-1}) d_r^+(v)~.
\]
Using the same argument as the one we used to bound~\eqref{eq:hyper},
\[
    \Pr\big(E_r(v) \mid \pi_r = v,\, H_{r-1}\big)
\ge
    1 - \left(1-\frac{d_r^+(v)}{n_r}\right)^{\ceil{f(n_r)}}
\ge
    1 - \left(1-\frac{1}{\ceil{f(n_r)}}\right)^{\ceil{f(n_r)}}
\ge
    1 - \frac{1}{e}
\]
and $\Pr(\pi_r = v \mid H_{r-1}) = \frac{1}{n_r+1}$ for any $v \in V_r$, we may write
\[
    \E[\delta_r \mid H_{r-1}]
\ge
    \left(1-\frac{1}{e}\right)\frac{\E[S_r \mid H_{r-1}]}{n_r+1}
\ge
    \left(1-\frac{1}{e}\right)\frac{\E[S_r \mid H_{r-1}]}{n}~.
\]
Observe now that $\sum_{r=1}^{\ceil{f(n-1)}} \delta_r \le n_1 - \nfin \le n-1$ and $S_r$ is monotonically nonincreasing in $r$. Thus 
\[
    n-1
\ge
    \sum_{r=1}^{\ceil{f(n-1)}} \E[\delta_r]
\ge
    \frac{1}{n} \left(1-\frac{1}{e}\right) \sum_{r=1}^{\ceil{f(n)}} \E[S_r]
\ge
    \frac{\ceil{f(n-1)}}{n} \left(1-\frac{1}{e}\right) \E[\Sfin]
\]
which implies $\E[\Sfin] \le \big(\frac{e}{e-1}\big)\frac{n(n-1)}{\ceil{f(n-1)}} \le \big(\frac{e}{e-1}\big)\frac{n(n-1)}{f(n-1)}$.
By the properties of $f$, however, $\big(\frac{e}{e-1}\big)\frac{n(n-1)}{f(n-1)} \le \big(\frac{e}{e-1}\big)\frac{n^2}{{f(n)}}$. So we have
\begin{align*}
    \sum_{e \in \scE} \Pr\big(B_3(e)\big)
\le
    \frac{1}{2} \left( \sum_{u \in \Vfin} \E\left[\frac{\nfin}{f(\nfin)}\right] + \E[\Sfin] \right)
\le
    \frac{1}{2} \left( \frac{n^2}{f(n)} + \frac{e}{e-1} \frac{n^2}{f(n)} \right)
\end{align*}
as claimed.

\paragraph{Bounding the number of queries.}
In any given round, $\access$ asks less than $n$ queries.
Since the number of rounds is at most $\lceil f(n) \rceil$, the overall number of queries is less than $n \lceil f(n) \rceil$.

\paragraph{\kc\ as special case.}
One can immediately see that, if $f(n)=n$ for all $n$, then \access\ coincides with \kc\ and therefore the bound $\E[\Delta] \le 3\OPT$ applies~\cite{Ailon2008}.

%TH:2**************************************************************************************
%******************************************************************************************
%******************************************************************************************
%******************************************************************************************

\subsection{Pseudocode of \aggress}
\setcounter{algorithm}{1}
\renewcommand{\algorithmicrequire}{\textbf{Parameters:}}
\begin{algorithm}[h!]
\caption{
\label{alg:access2}
\aggress\ with query rate $f$}
\begin{algorithmic}[1]
\setcounter{ALG@line}{0}
\Require{
residual node set $V_r$, round index $r$
}
\If{$\binom{|V_r|}{2} \le 2 n^2/f(n)$} STOP and declare every $v \in V_r$ as singleton\label{line:agg:stop1}
\EndIf
\State Sample the labels of $\ceil{\binom{|V_r|}{2} f(n)/n^2}$ pairs chosen u.a.r.\ from $\binom{V_r}{2}$ \label{line:agg:samp1}
\If{no label is positive}\label{line:agg:checkstop}
\State STOP and declare every $v \in V_r$ as singleton\label{line:agg:stop2}
\EndIf
\State Draw pivot $\pi_r$ u.a.r.\ from $V_r$
\State $C_r \gets \{\pi_r\}$ \Comment{Create new cluster and add the pivot to it}
\State Draw a random subset $S_r$ of $\ceil{f(|V_r|-1)}$ nodes from $V_r\setminus\{\pi_r\}$ 
\For{each $u \in S_r$} query $\sigma(\pi_r,u)$ \label{line:agg_q1}
\EndFor
\If{$\exists\,u \in S_r$ such that $\sigma(\pi_r,u)=+1$} %\label{line:sing}
\Comment{Check if there is at least an edge}
    \State Query all remaining pairs $(\pi_r,u)$ for $u \in V_r\setminus\big(\{\pi_r\}\cup S_r\big)$ \label{line:agg_q2}
    \State $C_r \gets C_r \cup \theset{u}{\sigma(\pi_r,u) = +1}$  \Comment{Populate cluster based on queries}
\EndIf
\State Output cluster $C_r$
\State $\aggress(V_r \setminus C_r, r+1)$ \Comment{Recursive call on the remaining nodes}
\end{algorithmic}
\end{algorithm}

\subsection{Proof of Theorem 2}
We refer to the pseudocode of \aggress\ (Algorithm~\ref{alg:access2}).

\paragraph{Bounding $\E[\Delta_A]$.}
Let $G_r$ be the residual graph at round $r$.
The total clustering cost $\Delta_A$ of \aggress\ can be bounded by the sum of two terms: the clustering cost $\Delta_1$ of \access\ without round restriction (i.e.\ \access\ terminating only when the residual graph is empty), and the number of edges $\Delta_2$ in the residual graph $G_r$ if $r$ is the round at which \aggress\ stops.
Concerning $\Delta_1$, the proof of Theorem 1 shows that $\E[\Delta_1] \le 3\OPT + n/e$.
Concerning $\Delta_2$, we have two cases. If \aggress\ stops at line~\ref{line:agg:stop1}, then obviously $\Delta_2 \le 2 n^2/f(n)$.
If instead \aggress\ stops at line~\ref{line:agg:stop2}, then note that for any $k \ge 0$ the probability that such an event happens given that $\Delta_2=k$ is at most:
\[
\left(1 - \frac{k}{\binom{|V_r|}{2}}\right)^{\big\lceil\binom{|V_r|}{2} f(n)/n^2\big\rceil} \le e^{-k f(n)/n^2} 
\]
Thus $\E[\Delta_2] \le \max_{k \ge 1}(k e^{-k f(n)/n^2}) \le \frac{n^2}{ef(n)} < 2 n^2/f(n)$.

\paragraph{Bounding $\E[Q]$.}
The queries performed at line 1 are deterministically at most $n \ceil{f(n)}$.
Concerning the other queries (line~\ref{line:agg_q1} and line~\ref{line:agg_q2}), we divide the algorithm in two phases: the ``heavy'' rounds $r$ where $G_r$ still contains at least $n^2/(2f(n))$ edges, and the remaining ``light'' rounds where $G_r$ contains less than $n^2/(2f(n))$ edges.

Consider first a ``heavy'' round $r$.
We see $G_r$ as an arbitrary fixed graph: for all random variables mentioned below, the distribution is thought solely as a function of the choices of the algorithm in the current round (i.e., the pivot node $\pi_r$ and the queried edges).
Now, let $Q_r$ be the number of queries performed at lines~\ref{line:agg_q1} and~\ref{line:agg_q2}), and $R_r = |V_r|-|V_{r+1}|$ be the number of nodes removed.
Let $\pi_r$ be the pivot, and let $D_r$ be its degree in $G_r$.
Let $X_r$ be the indicator random variable of the event that $\sigma(\pi_r, u) = +1$ for some $u \in S_r$.
Observe that:
\begin{align*}
Q_r \le \ceil{f(|V_r|-1)} + X_r (|V_r|-1)
\qquad\text{and}\qquad
R_r = 1 + X_r \, D_r
\end{align*}
Thus $\E[Q_r] \le \ceil{f(|V_r|-1)} + \E[X_r] |V_r|$, while $\E[R_r] = 1 + \E[X_r D_r]$.
However, $X_r$ is monotonically increasing in $D_r$, so $\E[X_r D_r] = \E[X_r]\E[D_r] + \operatorname{Cov}(X_r,D_r) \ge \E[X_r]\E[D_r]$.
Moreover, by hypothesis $\E[D_r] \ge 2\big(n^2/(2f(n))\big)/|V_r| \ge n/f(n)$.
Thus:
\begin{align*}
\E[R_r] &\ge 1 + \E[X_r] \E[D_r]
\\ &\ge 1 + \E[X_r] \frac{n}{f(n)}
\\ &\ge 1 + \E[X_r] \frac{|V_r|}{f(|V_r|)}
\\ &\ge 1 + \E[X_r] \frac{|V_r|}{\ceil{f(|V_r|)}}
\\ &\ge \frac{\E[Q_r]}{\ceil{f(|V_r|)}}
\\ &\ge \frac{\E[Q_r]}{\ceil{f(n)}}
\end{align*}
But then, since obviously $\sum_{r} R_r \le n$:
\begin{align*}
\E\left[\sum_{r \text{ heavy}}\!\! Q_r\right] \le \ceil{f(n)} \E\left[\sum_{r \text{ heavy}} \!\!R_r\right] \le n \ceil{f(n)}
\end{align*}
Consider now the ``light'' rounds, where $G_r$ contains less than $n^2/(2f(n))$ edges.
In any such round the expected number of edges found at line~\ref{line:agg:samp1} is less than:
\begin{align}
\frac{n^2/f(n)}{2\binom{|V_r|}{2}} \left\lceil\binom{|V_r|}{2} f(n)/n^2\right\rceil \label{eqn:expedg}
\end{align}
However, $\binom{|V_r|}{2} > 2n^2/f(n)$ otherwise \aggress\ would have stopped at line~\ref{line:agg:stop1}, hence:
\begin{align}
\left\lceil\binom{|V_r|}{2} f(n)/n^2\right\rceil \le \frac{3}{2}\binom{|V_r|}{2} f(n)/n^2
\end{align}
which implies that the expression in~\eqref{eqn:expedg} is bounded by $\frac{3}{4}$.
By Markov's inequality this is also an upper bound on the probability that \aggress\ finds some edge at line~\ref{line:agg:samp1}, so in every light round \aggress\ stops at line~\ref{line:agg:stop2} with probability at least $\frac{1}{4}$.
Hence \aggress\ completes at most $4$ light rounds in expectation; the corresponding expected number of queries is then at most $4n$.

%TH:3**************************************************************************************
%******************************************************************************************
%******************************************************************************************
%******************************************************************************************
\subsection{Proof of Theorem 3}
First of all, note that if the residual graph $G_r$ contains $\scO(n^2/f(n))$ edges, from $r$ onward \aggress\ stops at each round independently with constant probability.
The expected number of queries performed before stopping is therefore $\scO(n)$, and the expected error incurred is obviously at most $\scO(n^2/f(n))$.

We shall then bound the expected number of queries required before the residual graph contains $\scO(n^2/f(n))$ edges.
In fact, by definition of $i'$, if \aggress\ removes $C_{i'},\ldots,C_{\ell}$, then the residual graph contains $\scO(n^2/f(n))$ edges.
We therefore bound the expected number of queries before $C_{i'},\ldots,C_{\ell}$ are removed.

First of all recall that, when pivoting on a cluster of size $c$, the probability that the cluster is \emph{not} removed is at most $e^{-cf(n)/n}$.
Thus the probability that the cluster is not removed after $\Omega(c)$ of its nodes have been used as pivot is $e^{-\Omega(c^2) f(n)/n}$.
Hence the probability that \emph{any} of $C_{i'},\ldots,C_{\ell}$ is not removed after $\Omega(c)$ of its nodes are used as pivot is, setting $c=\Omega\big(h(n)\big)$ and using a union bound, at most $p = n e^{-\Omega(h(n)^2) f(n)/n}$.
Observe that $h(n) = \Omega\big(n/f(n)\big)$, for otherwise $\sum_{j=1}^{i'} \binom{C_j}{2}=o\big(n^2/f(n)\big)$, a contradiction.
Therefore $p \le n e^{-\Omega(h(n))}$.
Note also that we can assume $h(n) = \omega(\ln n)$, else the theorem bound is trivially $O(n^2)$.
%Note that we can , if $h(n) = O(\lg(n))$, then the bound is $O(n^2)$.
This gives $p = \scO\big(n e^{-\omega(\ln n)}\big) = o\big(1/\operatorname{poly}(n)\big)$.
%$ \le O(n/f(n)) e^{-\Omega(n/f(n))}$.
We can thus condition on the events that, at any point along the algorithm, every cluster among $C_{i'},\ldots,C_{\ell}$ that is still in the residual graph has size $\Omega\big(h(n)\big)$; the probability of any other event changes by an additive $\scO(p)$, which can be ignored.

Let now $k = \ell - i' + 1$, and suppose at a generic point $k' \le k$ of the clusters $C_{i'},\ldots,C_{\ell}$  are in the residual graph.
Their total size is therefore $\Omega\big(k' h(n)\big)$.
Therefore $\scO\big(n/k'h(n)\big)$ rounds in expectation are needed for the pivot to fall among those clusters.
Each time this happens, with probability $1 - e^{-\Omega(h(n))f(n)/n} = \Omega(1)$ the cluster containing the pivot is removed.
Hence, in expectation a new cluster among $C_{i'},\ldots,C_{\ell}$ is removed after $\scO\big(n/k'h(n)\big)$ rounds.
By summing over all values of $k'$, the number of expected rounds to remove all of $C_{i'},\ldots,C_{\ell}$  is
\begin{align*}
\scO\left(\sum_{k'=1}^k \frac{n}{k' h(n)}\right) = \scO\big(n (\ln n) / h(n)\big)
\end{align*}
Since each round involves $\scO(n)$ queries, the bound follows.

\section{Supplementary Material for Section 4}
\label{s:crec}

\subsection{Proof of Theorem~4}
%TH:4**************************************************************************************
%******************************************************************************************
%******************************************************************************************
%******************************************************************************************
%\begin{proof}
Fix any $C$ that is $(1-\epsilon)$-knit.
We show that \access\ outputs a $\hat{C}$ such that
\begin{equation}
\label{eq:toprove}
    \E\big[|\hat{C} \cap C|\big] \ge \max\left\{\left(1-\frac{5}{2}\epsilon\right)|C| - 2 \frac{n}{f(n)}, \left(\frac{f(n)}{n}-\frac{5}{2}\epsilon\right)|C|\right\}
\;\text{and}\;
    \E\big[|\hat{C} \cap \bar{C}|\big] \le \frac{\epsilon}{2}|C|
\end{equation}
One can check that these two conditions together imply the first two terms in the bound.
We start by deriving a lower bound on $\E\big[|\hat{C} \cap C|\big]$ for \kc\ assuming $|E_C| = \binom{|C|}{2}$.
Along the way we introduce most of the technical machinery.
We then port the bound to \access, relax the assumption to $|E_C| \ge (1-\epsilon)\binom{|C|}{2}$, and bound $\E\big[|\hat{C} \cap \bar{C}|\big]$ from above.
Finally, we add the $|C|e^{-|C|f(n)/5n}$ part of the bound.
To lighten the notation, from now on $C$ denotes both the cluster and its cardinality $|C|$.

For the sake of analysis, we see \kc\ as the following equivalent process.
First, we draw a random permutation $\pi$ of $V$.
This is the ordered sequence of \emph{candidate pivots}.
Then, we set $G_1 = G$, and for each $i=1,\ldots,n$ we proceed as follows.
If $\pi_i \in G_i$, then $\pi_i$ is used as an actual pivot; in this case we let $G_{i+1} = G_i \setminus (\pi_i \cup \scN_{\pi_i})$ where $\scN_v$ is the set of neighbors of $v$.
If instead $\pi_i \notin G_i$, then we let $G_{i+1}=G_{i}$.
Hence, $G_i$ is the residual graph just before the $i$-th candidate pivot $\pi_i$ is processed.
We indicate the event $\pi_i \in G_i$ by the random variable $P_i$:
\begin{align}
P_i = \Ind{\pi_i \in G_i} = \Ind{\pi_i \text{ is used as pivot}}
\end{align}
More in general, we define a random variable indicating whether node $v$ is ``alive'' in $G_i$: 
\begin{align}
\label{eqn:Xvi}
X(v,i) = \Ind{v \in G_i} = \Ind{v \notin \cup_{j < i \,:\, P_j=1} \, (\pi_j \cup \scN_{\pi_j})}
\end{align}
Let $i_C = \min\{i : \pi_i \in C\}$ be the index of the first candidate pivot of $C$.
Define the random variable:
\begin{align}
\label{eqn:S0}
S_C = |C \cap G_{{i_C}}| = \sum_{v \in C} X(v,i_C)
\end{align}
In words, $S_C$ counts the nodes of $C$ still alive in $G_{i_C}$.
Now consider the following random variable:
\begin{align}
\label{eqn:S}
S = P_{i_C} \cdot S_C
\end{align}
Let $\hat{C}$ be the cluster that contains $\pi_{i_C}$ in the output of \kc.
It is easy to see that $|C \cap \hat{C}| \ge S$.
Indeed, if $P_{i_C}=1$ then $\hat{C}$ includes $C \cap G_{{i_C}}$, so $|C \cap \hat{C}| \ge P_{i_C} S_C = S$.
If instead $P_{i_C}=0$, then $S=0$ and obviously $|C \cap \hat{C}| \ge 0$.
Hence in any case $|C \cap \hat{C}| \ge S$, and $\E\big[|C \cap \hat{C}|\big] \ge \E[S]$.
Therefore we can bound $\E\big[|C \cap \hat{C}|\big]$ from below by bounding $\E[S]$ from below.

Before continuing, we simplify the analysis by assuming \kc\ runs on the graph $G$ after all edges not incident on $C$ have been deleted.
We can easily show that this does not increase $S$.
First, by~\eqref{eqn:Xvi} each $X(v,i_C)$ is a nonincreasing function of $\theset{P_i}{i < i_C}$.
Second, by~\eqref{eqn:S0} and~\eqref{eqn:S}, $S$ is a nondecreasing function of $\theset{X(v,i_C)}{v\in C}$.
Hence, $S$ is a nonincreasing function of $\theset{P_i}{i < i_C}$.
Now, the edge deletion forces $P_i=1$ for all $i < i_C$, since any $\pi_i:i < i_C$ has no neighbor $\pi_j : j < i$.
Thus the edge deletion does not increase $S$ (and, obviously, $\E[S]$).
We can then assume $G[V \setminus C]$ is an independent set.
At this point, any node not adjacent to $C$ is isolated and can be ignored.
We can thus restrict the analysis to $C$ and its neighborhood in $G$.
Therefore we let $\bar{C}=\theset{v}{\{u,v\} \in E,\, u \in C, v \notin C}$ denote both the neighborhood and the complement of $C$.

We turn to bounding $\E[S]$.
For now we assume $G[C]$ is a clique; we will then relax the assumption to $|E_C| \ge (1-\epsilon)\binom{C}{2}$.
Since by hypothesis $\cut(C, \bar{C}) < \epsilon C^2$, the average degree of the nodes in $\bar{C}$ is less than $\epsilon C^2 / \bar{C}$.
This is also a bound on the expected number of edges between $C$ and a node drawn u.a.r.\ from $\bar{C}$.
But, for any given $i$, conditioned on $i_C-1=i$ the nodes $\pi_{1},\ldots,\pi_{i_C-1}$ are indeed drawn u.a.r.\ from $\bar{C}$, and so have a total of at most $i \epsilon C^2 / \bar{C}$ edges towards $C$ in expectation.
Thus, over the distribution of $\pi$, the expected number of edges between $C$ and $\pi_{1},\ldots,\pi_{i_C-1}$ is at most:
\begin{align}
\sum_{i=0}^n \frac{i \epsilon C^2}{\bar{C}} \Pr(i_C-1 = i) = \frac{\epsilon C^2}{\bar{C}} \E[i_C-1]
=  \frac{\epsilon C^2}{\bar{C}} \frac{\bar{C}}{C+1} < \epsilon  C \label{eqn:degbound}
\end{align}
where we used the fact that $\E[i_C-1]=\bar{C}/(C+1)$.
Now note that (\ref{eqn:degbound}) is a bound on $C-\E[S_C]$, the expected number of nodes of $C$ that are adjacent to $\pi_1,\ldots,\pi_{i_C-1}$.
Therefore,
$
\E[S_C] \ge (1 - \epsilon)C
$.

Recall that $P_{i_C}$ indicates whether $\pi_{i_C}$ is not adjacent to any of $\pi_1,\ldots,\pi_{i_C-1}$. Since the distribution of $\pi_{i_C}$ is uniform over $C$, $\Pr(P_{i_C} \mid S_C) = S_C/C$.
But $S = P_{i_C}S_C$, hence $\E[S \mid S_C] = (S_C)^2/C$, and thus $\E[S] = \E\big[(S_C)^2\big]/C$.
Using $\E[S_C] \ge (1 - \epsilon)C$ and invoking Jensen's inequality we obtain
\begin{align}
\E[S] \ge \frac{\E[S_C]^2}{C} \ge (1-\epsilon)^2 C \ge (1-2\epsilon)C
\end{align}
which is our bound on $\E\big[|C \cap \hat{C}|\big]$ for \kc.

Let us now move to \access.
We have to take into account the facts that \access\ performs $f(|G_r|-1)$ queries on the pivot before deciding whether to perform $|G_r|-1$ queries, and that \access\ stops after $f(n-1)$ rounds.
We start by addressing the first issue, assuming for the moment \access\ has no restriction on the number of rounds.

Recall that $\Pr(P_{i_C} \mid S_C) = S_C/C$.
Now, if $P_{i_C}=1$, then we have $S_C-1$  edges incident on $\pi_{i_C}$.
It is easy to check that, if $n_r+1$ is the number of nodes at the round when $\pi_{i_C}$ is used, then the probability that \access\ finds some edge incident on $\pi_{i_C}$ is at least:
\begin{align}
1-\Big(1-\frac{S_C-1}{n_{r}}\Big)^{\ceil{f(n_r)}} \ge 1-e^{-f(n_r) \frac{S_C-1}{n_r}} \ge 1-e^{-f(n) \frac{S_C-1}{n}}
\end{align}
and, if this event occurs, then $S=S_C$.
Thus
\begin{align}
\E[S \mid S_C] = \Pr(P_{i_C} \mid S_C) S_C \ge \left(1-e^{-f(n) \frac{S_C-1}{n}}\right) \frac{S_C^2}{C} \label{eqn:ESgivenSC}
%\\ &\ge (S_C/C)\big(1-\frac{2n}{f(n)S_C}\big) S_C
\ge \frac{S_C^2}{C} - S_C\frac{2 n}{f(n)C}
%\\ &\ge (S_C^2)/C - 2\epsilon S_C
\end{align}
where we used the facts that for $S_C \le 1$ the middle expression in~(\ref{eqn:ESgivenSC}) vanishes, that $e^{-x}< 1/x$ for $x > 0$, and that $1/x < 2/(x+1)$ for all $x \ge 2$.
Simple manipulations, followed by Jensen's inequality and an application of $\E[S_C] \ge (1 - \epsilon)C$, give
\begin{align}
\E[S] \ge (1-\epsilon)^2 C - (1-\epsilon)C\frac{2 n }{ f(n) C} %\ge (1-4\epsilon)C
\ge (1-2\epsilon ) C - 2\frac{n }{ f(n)}
%\E[S] \ge \E[S_C](\E[S_C]/C - 2\epsilon) \ge (1-4\epsilon)C
\end{align}
We next generalize the bound to the case $E_C \ge (1-\epsilon)\binom{C}{2}$.
To this end note that, since at most $\epsilon\binom{C}{2}$ edges are missing from any subset of $C$, then any subset of $S_C$ nodes of $C$ has average degree at least
\begin{align}
\max\left\{0, S_C-1 - \binom{C}{2}\frac{2\epsilon}{S_C}\right\} \ge S_C-\frac{\epsilon C (C-1)}{2S_C} - 1
\end{align}
We can thus re-write~(\ref{eqn:ESgivenSC}) as
\begin{align}
\E[S \mid S_C] &\ge \frac{S_C}{C} \left(1-e^{-f(n) \frac{S_C-1}{n}}\right) \left(S_C-\frac{\epsilon C (C-1)}{2S_C}\right)
\end{align}
Standard calculations show that this expression is bounded from below by $\frac{S_C^2}{C} - S_C\frac{2 n }{ f(n) C} - \frac{\epsilon C}{2}$, which by calculations akin to the ones above leads to $\E[S] \ge (1-\frac{5}{2}\epsilon)C - 2\frac{n}{f(n)}$.

Similarly, we can show that $\E[S] \ge \big(\frac{f(n)}{n} -\frac{5}{2}\epsilon\big)C$.
To this end note that when \access\ pivots on $\pi_{i_C}$ all the remaining cluster nodes are found with probability at least $\frac{f(n)}{n}$ (this includes the cases $S_C \le 1$, when such a probability is indeed $1$).
In~\eqref{eqn:ESgivenSC}, we can then replace $1-e^{-f(n) \frac{S_C-1}{n}}$ with $\frac{f(n)}{n}$, which leads to $\E[S] \ge \big(\frac{f(n)}{n} -\frac{5}{2}\epsilon\big)C$.
This proves the first inequality in \eqref{eq:toprove}.

For the second inequality in \eqref{eq:toprove}, note that any subset of $S_C$ nodes has $\cut(C,\bar{C}) \le \epsilon \binom{C}{2}$. Thus, $\pi_{i_C}$ is be incident to at most $\frac{\epsilon}{S_C}\binom{C}{2}$ such edges in expectation.
The expected number of nodes of $\bar{C}$ that \access\ assigns to $\hat{C}$, as a function of $S_C$, can thus be bounded by $\frac{S_C}{C} \frac{\epsilon}{S_C}\binom{C}{2} < \frac{\epsilon}{2}C$.

As far as the $\scO(C e^{-C f(n)/n})$ part of the bound is concerned,
simply note that the bounds obtained so far hold unless $i_C > \ceil{f(n-1)}$, in which case \access\ stops before ever reaching the first node of $C$.
If this happens, $\hat{C}=\{\pi_{i_C}\}$ and $|\hat{C} \oplus C| < |C|$.
The event $i_C > \ceil{f(n-1)}$ is the event that no node of $C$ is drawn when sampling $\ceil{f(n-1)}$ nodes from $V$ without replacement.
We can therefore apply Chernoff-type bounds to the random variable $X$ counting the number of draws of nodes of $C$ and get $\Pr\big(X < (1-\beta)\E[X]) \le \exp(-\beta^2\E[X]/2\big)$ for all $\beta > 0$.
%=f(n-1)$ is the number of draws.
In our case $\E[X] = \ceil{f(n-1)}|C|/n$, and we have to bound the probability that $X$ equals $0 < (1-\beta)\E[X]$.
Thus
\begin{align*}
\Pr(X = 0) \le \exp\left(-\frac{\beta^2 \E[X]}{2}\right)
=
\exp\left(-\frac{\beta^2 \ceil{f(n-1)}|C|}{2 n}\right)
\end{align*}
Note however that $\ceil{f(n-1)} \ge f(n)/2$ unless $n=1$ (in which case $V$ is trivial).
Then, choosing e.g.\ $\beta > \sqrt{4/5}$ yields $\Pr(X = 0) < \exp\big(-{|C|f(n)/5n}\big)$.
This case therefore adds at most $|C|\exp(-{|C|f(n)/5n})$ to $\E[|\hat{C} \oplus C|]$.

%TH:5**************************************************************************************
%******************************************************************************************
%******************************************************************************************
%******************************************************************************************
\subsection{Proof of Theorem~5}
Before moving to the actual proof, we need some ancillary results.
The next lemma bounds the probability that \access\ does not pivot on a node of $C$ in the first $k$ rounds.
\begin{lemma}
\label{lem:bound_late}
Fix a subset $C \subseteq V$ and an integer $k \ge 1$, and let $\pi_1,\ldots,\pi_n$ be a random permutation of $V$.
For any $v \in C$ let $X_v = \Ind{v \in \{\pi_1,\ldots,\pi_k\}}$, and let $X_C = \sum_{v \in C} X_v$.
Then $\E[X_C]=\frac{k|C|}{n}$, and $\prob(X_C = 0) < e^{-\frac{k |C|}{3 n}}$.
\end{lemma}
\begin{proof}
Since $\pi$ is a random permutation, then for each $v \in C$ and each each $i=1,\ldots,k$ we have $\prob(\pi_i = v) = \frac{1}{n}$.
Therefore $\E[X_v] = \frac{k}{n}$ and $\E[X_C]=\frac{k |C|}{n}$.
Now, the process is exactly equivalent to sampling without replacement from a set of $n$ items of which $|C|$ are marked.
Therefore, the $X_v$'s are non-positively correlated and we can apply standard concentration bounds for the sum of independent binary random variables.
In particular, for any $\eta \in (0,1)$ we have:
\[
\prob(X_C = 0) \le \prob(X_C < (1-\eta)\E[X_C]) < \exp\Big(-\frac{\eta^2 \E[X_C]}{2}\Big)
\]
which drops below $e^{- \frac{k |C|}{3n}}$ by replacing $\E[X_C]$ and choosing $\eta \ge \sqrt{2/3}$.
\end{proof}
The next lemma is the crucial one.
\begin{lemma}
\label{lem:majority}
Let $\epsilon \le \frac{1}{10}$. Consider a strongly $(1-\epsilon)$-knit set $C$ with $|C| > \frac{10 n}{f(n)}$.
Let $u_C = \min\{v \in C\}$ be the id of $C$.
Then, for any $v \in C$, in any single run of \access\ we have $\prob(\id(v)=u_C) \ge \frac{2}{3}$.
\end{lemma}
\begin{proof}
We bound from above the probability that any of three ``bad'' events occurs.
As in the proof of Theorem 4, we equivalently see \access\ as going through a sequence of candidate pivots $\pi_1,\ldots,\pi_n$ that is a uniform random permutation of $V$.
Let $i_C = \min\{i : \pi_i \in C\}$ be the index of the first node of $C$ in the random permutation of candidate pivots.
The first event, $B_1$, is $\{i_C > \ceil{f(n-1)}\}$.
Note that, if $B_1$ does not occur, then \access\ will pivot on $\pi_{i_C}$.
The second event, $B_2$, is the event that $\pi_{i_C} \in V_{sing}$ if \access\ pivots on $\pi_{i_C}$ (we measure the probability of $B_2$ conditioned on $\bar{B_1}$).
The third event, $B_3$, is $\{\pi_{i_C} \notin P\}$ where $P = \scN_{u_C} \cap \scN_v$.
If none among $B_1,B_2,B_3$ occurs, then \access\ forms a cluster $\hat{C}$ containing both $u_C$ and $v$, and by the min-tagging rule sets $\id(v) = \min_{u \in \hat{C}} = u_C$.
We shall then show that $\prob(B_1 \cup B_2 \cup B_3) \le 1/3$.

For $B_1$, we apply Lemma~\ref{lem:bound_late} by observing that $i_C > \ceil{f(n-1)}$ corresponds to the event $X_C=0$ with $k=\ceil{f(n-1)}$.
Thus
\begin{align*}
    \prob(i_C > \ceil{f(n-1)}) < e^{-\frac{\ceil{f(n-1)}|C|}{3n}} \le
    e^{-\frac{f(n-1)}{3n}\frac{10\, n}{f(n)}} = e^{-\frac{f(n-1)}{f(n)}\frac{10 }{3}}
    <
e^{-3}
\end{align*}
where we used the fact that $n \ge |C| \ge 11$ and therefore $f(n-1)\ge \frac{10}{11}f(n)$.

For $B_2$, recall that by definition every $v \in C$ has at least $(1-\epsilon)c$ edges.
By the same calculations as the ones above, if \access\ pivots on $\pi_{i_C}$, then:
\begin{align*}
\prob(\pi_{i_C} \in V_{sing}) \le \exp\!\Big(\!-\frac{f(n-1)}{n-1}(1-\epsilon)c\!\Big) \le \exp\!\Big(\!-\frac{f(n-1)}{n-1}\big(1-\frac{1}{10}\big)\frac{10\, n}{f(n)}\!\Big) \le e^{-9} 
\end{align*}
For $B_3$, note that the distribution of $\pi_{i_C}$ is uniform over $C$.
Now, let $\scN_{u_C}$ and $\scN_v$ be the neighbor sets of $u_C$ and $v$ in $C$, and let $P=\scN_{u_C} \cap \scN_v$.
We call $P$ the set of good pivots.
Since $C$ is strongly $(1-\epsilon)$-knit, both $u_C$ and $v$ have at least $(1-\epsilon)c$ neighbors in $C$.
But then $|C \setminus P| \le 2\epsilon c$ and
\begin{align*}
\Pr(\pi_{i_C} \notin P) = \frac{|C \setminus P|}{|C|} \le 2\epsilon \le 1/5
\end{align*}
By a union bound, then, $\prob(B_1 \cup B_2 \cup B_3) \le e^{-3} + e^{-9} + 1/5 < 1/3$.
\end{proof}
We are now ready to conclude the proof.
Suppose we execute \access\ independently $K=48\lceil\ln(n/p)\rceil$ times with the min-tagging rule.
For a fixed $v \in G$ let $X_v$ be the number of executions giving $\id(v)=u_C$.
On the one hand, by Lemma~\ref{lem:majority}, $\E[X_v] \ge \frac{2}{3}K$.
On the other hand, $v$ will not be assigned to the cluster with id $u_C$ by the majority voting rule only if $X_v \le \frac{1}{2}K \le \E[X_v](1-\delta)$ where $\delta=\frac{1}{4}$.
By standard concentration bounds, then, $\Pr(X_v \le \frac{1}{2}K) \le \exp(-\frac{\delta^2 \E[X_v]}{2}) = \exp(-\frac{K}{48})$.
By setting $K=48\ln(n/p)$, the probability that $v$ is not assigned id $u_C$ is thus at most $p/n$.
A union bound over all nodes concludes the proof.

\section{Supplementary Material for Section 6}
\label{s:lb}
%TH:6**************************************************************************************
%******************************************************************************************
%******************************************************************************************
%******************************************************************************************

\subsection{Proof of Theorem 8}
We prove that there exists a distribution over labelings $\sigma$ with $\OPT=0$ on which any deterministic algorithm has expected cost at least $\frac{n\ve^2}{8}$. Yao's minimax principle then implies the claimed result.
 
Given $V = \{1,\ldots,n\}$, we define $\sigma$ by a random partition of the vertices in $d \ge 2$ isolated cliques $T_1,\ldots,T_d$ such that $\sigma(v,v') = +1$ if and only if $v$ and $v'$ belong to the same clique. The cliques are formed by assigning each node $v \in V$ to a clique $I_v$ drawn uniformly at random with replacement from $\{1,\dots,d\}$, so that $T_i = \theset{v \in V}{I_v=i}$. Consider a deterministic algorithm making queries $\{s_t,r_t\} \in \scE$. Let $E_i$ be the event that the algorithm never queries a pair of nodes in $T_i$ with $|T_i| \ge \frac{n}{2d} > 5$. Apply Lemma~\ref{lem:cluster} below with $d = \frac{1}{\ve}$. This implies that the expected number of non-queried clusters of size at least $\frac{n}{2d}$ is at least $\frac{d}{2} = \frac{1}{2\ve}$. The overall expected cost of ignoring these clusters is therefore at least
\[
	\frac{d}{2}\left(\frac{n}{2d}\right)^2 = \frac{n^2}{8d} = \frac{\ve n^2}{8}
\]
and this concludes the proof.

\begin{lemma}
\label{lem:cluster}
Suppose $d > 0$ is even, $n \ge 16d\ln d$, and $B < \frac{d^2}{50}$. Then for any deterministic learning algorithm making at most $B$ queries,
\[
    \sum_{i=1}^d \Pr(E_i) > \frac{d}{2}~.
\]
\end{lemma}
\begin{proof}
For each query $\{s_t,r_t\}$ we define the set $L_t$ of all cliques $T_i$ such that $s_t\not\in T_i$ and some edge containing both $s_t$ and a node of $T_i$ was previously queried. The set $R_t$ is defined similarly using $r_t$. Formally,
\begin{align*}
    L_t = & \theset{ i }{ (\exists \tau < t) \; s_\tau=s_t \,\wedge\, r_\tau\in T_i \,\wedge\, \sigma(s_{\tau},r_{\tau}) = -1 }
\\
    R_t = & \theset{ i }{ (\exists \tau < t) \; r_\tau=r_t \,\wedge\, s_\tau\in T_i \,\wedge\, \sigma(s_{\tau},r_{\tau}) = -1 }~.
\end{align*}
Let $D_t$ be the event that the $t$-th query discovers a new clique of size at least $\frac{n}{2d}$, and let $P_t = \max\bigl\{|L_t|,|R_t|\bigr\}$. Using this notation,
\begin{align}
\label{eq:clubound}
    \sum_{t=1}^B \Ind{D_t}
=
    \sum_{t=1}^B \Ind{D_t \,\wedge\, P_t < d/2} + \underbrace{\sum_{t=1}^B \Ind{D_t \,\wedge\, P_t \ge d/2}}_{N}~.
\end{align}
We will now show that unless $B \ge \tfrac{d^2}{50}$, we can upper bound $N$ deterministically by $\sqrt{2B}$.

Suppose $N > \tfrac{d}{2}$, and let $t_1,\dots,t_N$ be the times $t_k$ such that $\Ind{D_{t_k} \,\wedge\, P_{t_k} \ge d/2} = 1$. Now fix some $k$ and note that, because the clique to which $s_{t_k}$ and $r_{t_k}$ both belong is discovered, neither $s_{t_k}$ nor $r_{t_k}$ can occur in a future query $\{s_t,r_t\})$ that discovers a new clique. Therefore, in order to have $\Ind{D_t \,\wedge\, P_t \ge d/2} = 1$ for $N > \tfrac{d}{2}$ times, at least
\[
    \binom{N}{2} \ge \frac{d^2}{8}
\]
queries must be made, since each one of the other $N-1 \ge \frac{d}{2}$ discovered cliques can contribute with at most a query to making $P_t \ge \tfrac{d}{2}$. So, it takes at least
$
    B \ge \frac{d^2}{8}
$
queries to discover the first $\frac{d}{2}$ cliques of size at least two, which contradicts the lemma's assumption that $B \le \tfrac{d^2}{16}$. Therefore, $N \le \tfrac{d}{2}$.

Using the same logic as before, in order to have $\Ind{D_t \,\wedge\, P_t \ge d/2} = 1$ for $N \le \tfrac{d}{2}$ times, at least
\[
    \frac{d}{2} + \left(\frac{d}{2} - 1\right) + \dots + \left(\frac{d}{2} - N + 1\right)
\]
queries must be made. So, it must be
\[
    B \ge \sum_{k=1}^N \left(\frac{d}{2} - (k - 1)\right) = (d+1)\frac{N}{2} - \frac{N^2}{2}
\]
or, equivalently, $N^2 -(d+1)N + 2B \ge 0$. Solving this quadratic inequality for $N$, and using the hypothesis $N\le \tfrac{d}{2}$, we have that $N\leq \frac{(d+1)-\sqrt{(d+1)^2-8B}}{2}$. Using the assumption that $B\leq \tfrac{d^2}{50}$ we get that $N\leq \sqrt{2B}$.

We now bound the first term of~(\ref{eq:clubound}) in expectation. The event $D_t$ is equivalent to $s_t,r_t \in T_i$ for some $i \in \neg L_t \cap \neg R_t$, where for any $S \subseteq \{1,\dots,d\}$ we use $\neg S$ to denote $\{1,\dots,d\}\setminus S$.

Let $\Pr_t = \Pr\bigl(\,\cdot\mid P_t < d/2 \bigr)$. For $L',R'$ ranging over all subsets of $\{1,\dots,d\}$ of size strictly less than $\tfrac{d}{2}$,
\begin{align}
\nonumber
    \Pr_t(D_t)
&=
    \sum_{L',R'}\sum_{i \in \neg L' \cap \neg R'} \!\Pr_t\bigl(s_t \in T_i \,\wedge\, r_t \in T_i \,\big|\, L_t = L',\, R_t = R'\bigr)\,\Pr_t(L_t = L' \,\wedge\, R_t = R')
\\ &=
\label{eq:indep}
    \sum_{L',R'}\sum_{i \in \neg L' \cap \neg R'} \!\Pr_t\bigl(s_t \in T_i \,\big|\, L_t = L'\bigr)\,\Pr_t\bigl(r_t \in T_i \,\big|\, R_t = R'\bigr)\,\Pr_t(L_t = L' \,\wedge\, R_t = R')
\\ &=
\label{eq:split}
    \sum_{L',R'}\sum_{i \in \neg L' \cap \neg R'} \frac{1}{|\neg L'|}\,\frac{1}{|\neg R'|}\,\Pr_t(L_t = L' \,\wedge\, R_t = R')
\\ &=
\nonumber
    \sum_{L',R'} \frac{|\neg L' \cap \neg R'|}{|\neg L'|\,|\neg R'|}\,\Pr_t(L_t = L' \,\wedge\, R_t = R')
\\ &\le
\label{eq:count}
	\frac{2}{d}~.
\end{align}
Equality~(\ref{eq:indep}) holds because $P_t = \max\{L_t,R_t\} < \frac{d}{2}$ implies that there are at least two remaining cliques to which $s_t$ and $r_t$ could belong, and each node is independently assigned to one of these cliques. Equality~(\ref{eq:split}) holds because, by definition of $L_t$, the clique of $s_t$ is not in $L_t$, and there were no previous queries involving $s_t$ and a node belonging to a clique in $\neg L_t$ (similarly for $r_t$). Finally, (\ref{eq:count}) holds because $|\neg L'| \ge \tfrac{d}{2}$, $|\neg R'| \ge \tfrac{d}{2}$, and $|\neg L' \cap \neg R'| \le \min\{|\neg L'|,|\neg R'|\}$. Therefore,
\begin{align*}
    \sum_{t=1}^B \Pr\bigl(D_t \,\wedge\, P_t < d/2\bigr)
\le
    \sum_{t=1}^B \Pr\bigl(D_t \mid P_t < d/2\bigr)
\le
    \frac{2B}{d}~.
\end{align*}
Putting everything together,
\begin{equation}\label{eq:clubound1}
\E\left[\sum_{t=1}^B \Ind{D_t}\right] \leq \frac{2B}{d} + \sqrt{2B}~.
\end{equation}
On the other hand, we have
\begin{equation}
\label{eq:clubound2}
	\sum_{t=1}^B \Ind{D_t}
=
	\sum_{i=1}^{d}\Big(\Ind{|T_i|\ge\tfrac{n}{2d}}-\Ind{E_i}\Big)
=
	d - \sum_{i=1}^{d}\Big(\Ind{|T_i|<\tfrac{n}{2d}}+\Ind{E_i}\Big)
\end{equation}
Combining \eqref{eq:clubound1} and \eqref{eq:clubound2}, we get that
\[
	\sum_{i=1}^{d}\Pr(E_i)
\ge
	d-\sum_{i=1}^{d}\Pr\big(|T_i|<\tfrac{n}{2d}\big) - \frac{2B}{d}-\sqrt{2B}~.
\]
By Chernoff-Hoeffding bound,
$
	\Pr\big(|T_i|<\tfrac{n}{2d}\big) \le \frac{1}{d^2}
$
for each $i=1,\dots,d$ when $n \ge 16d\ln d$. Therefore,
\[
	\sum_{i=1}^{d}\Pr(E_i)
\ge
	d - \frac{2B+1}{d} - \sqrt{2B}~.
\]
To finish the proof, suppose on the contrary that $\sum_{i=1}^{d}\Pr(E_i)\leq \frac{d}{2}$. Then from the inequality above, we
would get that
\[
	\frac{d}{2}
\ge
	d - \frac{2B+1}{d} - \sqrt{2B}
\]
which implies $B\geq \left(\frac{2-\sqrt{2}}{4}\right)^2 d^2> \frac{d^2}{50}$, contradicting the assumptions. Therefore, we must have $\sum_{i=1}^{d}\Pr(E_i)> \frac{d}{2}$ as required.
\end{proof}

\subsection{Proof of Theorem 9}
Choose a suitably large $n$ and let $V=[n]$.
We partition $V$ in two sets $A$ and $B$, where $|A|=\alpha n$ and $|B|=(1-\alpha)n$; we will eventually set $\alpha = 0.9$, but for now we leave it free to have a clearer proof.
The set $A$ is itself partitioned into $k = 1/\epsilon$ subsets $A_1,\ldots,A_k$, each one of equal size $\alpha n/k$ (the subsets are not empty because of the assumption on $\ve$).
The labeling $\sigma$ is the distribution defined as follows.
For each $i=1,\ldots,k$, for each pair $u,v \in A_i$, $\sigma(u,v)=+1$; for each $u,v\in B$, $\sigma(u,v) = -1$.
Finally, for each $v \in B$ we have a random variable $i_v$ distributed uniformly over $[k]$.
Then, $\sigma(u,v)=+1$ for all $u \in A_{i_v}$ and $\sigma(u,v)=-1$ for all $u \in A \setminus A_{i_v}$.
Note that the distribution of $i_v$ is independent of the (joint) distributions of the $i_{w}$'s for all $w \in B \setminus \{v\}$.

Let us start by giving an upper bound on $\E[\OPT]$.
To this end consider the (possibly suboptimal) clustering $\cs = \{C_i : i \in [k]\}$ where $C_i=A_i \cup \{v \in B : i_v = i\}$.
One can check that $\cs$ is a partition of $V$.
The expected cost $\E[\Delta_{\cs}]$ of $\cs$ can be bound as follows.
First, note the only mistakes are due to pairs $u,v \in B$.
However, for any such fixed pair $u,v$, the probability of a mistake (taken over $\sigma$) is $\prob(i_u \ne i_v) = 1/k$.
Thus,
\begin{align}
\label{eqn:OPTbound}
\E[\OPT] \le \E[\Delta_0] < \frac{|B|^2}{k} = \frac{(1-\alpha)^2n^2}{k}
\end{align}
Let us now turn to the lower bound on the expected cost of the clustering produced by an algorithm.
For each $v \in B$ let $Q_v$ be the total number of distinct queries the algorithm makes to pairs $\{u,v\}$ with $u \in A$ and $v \in B$.
Let $Q$ be the total number of queries made by the algorithm; obviously, $Q \ge \sum_{v \in B} Q_v$.
Now let $S_v$ be the indicator variable of the event that one of the queries involving $v$ returned $+1$.
Both $Q_v$ and $S_v$ as random variables are a function of the input distribution and of the choices of the algorithm.
The following is key:
\begin{align}
\label{eqn:joint_small}
\prob(S_v \et Q_v < \nicefrac{k}{2}) < \frac{1}{2}
\end{align}
The validity of (\ref{eqn:joint_small}) is seen by considering the distribution of the input limited to the pairs $\{u,v\}$.
Indeed, $S_v \et Q_v < \nicefrac{k}{2}$ implies the algorithm discovered the sole positive pair involving $v$ in less than $k/2$ queries.
Since there are $k$ pairs involving $v$, and for any fixed $j$ the probability (taken over the input) that the algorithm finds that particular pair on the $j$-th query is exactly $1/k$.
Now,
\begin{align}
\prob(S_v \et Q_v < \nicefrac{k}{2}) + \prob(\bar{S_v} \et Q_v < \nicefrac{k}{2}) + \prob(Q_v \ge \nicefrac{k}{2}) = 1
\end{align}
and therefore
\begin{align}
\label{eqn:sumup}
\prob(\bar{S_v} \et Q_v < \nicefrac{k}{2}) + \prob(Q_v \ge \nicefrac{k}{2}) > \frac{1}{2}
\end{align}
Let us now consider $R_v$, the number of mistakes involving $v$ made by the algorithm.
We analyse $\E[R_v \,|\, \bar{S_v} \et Q_v < \nicefrac{k}{2}]$.
For all $i\in [k]$ let $Q_v^i$ indicate the event that, for some $u \in A_i$, the algorithm queried the pair $\{u,v\}$.
Let $I = \{i \in [k] : Q_v^i = 0\}$; thus $I$ contains all $i$ such that the algorithm did not query any pair $u,v$ with $u \in A_i$.
Suppose now the event $\bar{S_v} \et Q_v < \nicefrac{k}{2}$ occurs.
On the one hand, $\bar{S_v}$ implies that: %we have that:
\begin{align}
\Pr(\sigma(u,v)=+1 \mid I) = \left\{
\begin{array}{ll}
\nicefrac{1}{|I|} & u \in A_i, i \in I \\
0 & u \in A_i, i \in [k] \setminus I
\end{array}
 \right.
\end{align}
Informally speaking, this means that the random variable $i_v$ is distributed uniformly over the (random) set $I$.
Now observe that, again conditioning on the joint event $\bar{S_v} \et Q_v < \nicefrac{k}{2}$, whatever label $s$ the algorithm assigns to a pair $u,v$ with $u \in A_i$ where $i \in I$, the distribution of $\sigma(u,v)$ is independent of $s$.
This holds since $s$ can obviously be a function only of $I$ and of the queries made so far, all of which returned $-1$, and possibly of the algorithm's random bits.
In particular, it follows that:
\begin{align}
\Pr(\sigma(u,v) \ne s \mid I) \ge \min\big\{\nicefrac{1}{|I|}, 1-\nicefrac{1}{|I|}\big\}
\end{align}
However, $Q_v < \nicefrac{k}{2}$ implies that $|I| \ge k-Q_v > \nicefrac{k}{2} = \nicefrac{2}{\epsilon} > 2$, which implies $\min\{\nicefrac{1}{|I|}, 1-\nicefrac{1}{|I|}\} \ge \nicefrac{1}{|I|}$.
Therefore, $\Pr(\sigma(u,v) \ne s \mid I) \ge \nicefrac{1}{|I|}$ for all $u \in A_i$ with $i \in I$.

We can now turn to back to $R_v$, the number of total mistakes involving $v$.
Clearly, $R_v \ge \sum_{i = 1}^{k}\sum_{u \in A_i} \Ind{\sigma(u,v) \ne s}$.
Then:
\begin{align}
\E[R_v \,|\,E] &= \E\Big[\sum_{i = 1}^{k}\sum_{u \in A_i} \Ind{\sigma(u,v) \ne s} \,\Big|\, \bar{S_v} \et Q_v < \nicefrac{k}{2} \Big] \\
&= \E\Big[ \E\Big[ \sum_{i = 1}^{k}\sum_{u \in A_i} \Ind{\sigma(u,v) \ne s} \,\Big|\, I \Big] \,\Big|\, \bar{S_v} \et Q_v < \nicefrac{k}{2} \Big] \\
&\ge \E\Big[ \E\Big[ \sum_{i \in I}\sum_{u \in A_i} \Ind{\sigma(u,v) \ne s} \,\Big|\, I \Big] \,\Big|\, \bar{S_v} \et Q_v < \nicefrac{k}{2} \Big] \\
&\ge \E\Big[ \E\Big[ \sum_{i \in I}\sum_{u \in A_i} \frac{1}{|I|} \,\Big|\, I \Big] \,\Big|\, \bar{S_v} \et Q_v < \nicefrac{k}{2} \Big] \\
&= \E\Big[ \E\Big[ \frac{\alpha n}{k} \Big] \,\Big|\, \bar{S_v} \et Q_v < \nicefrac{k}{2} \Big] \\
&=  \frac{\alpha n}{k}
\end{align}
And therefore:
\begin{align*}
\E[R_v] &\ge \E[R_v \,|\, \bar{S_v} \et Q_v < \nicefrac{k}{2}] \cdot \prob(\bar{S_v} \et Q_v < \nicefrac{k}{2}) \\
&> \frac{\alpha n}{k} \cdot \prob(\bar{S_v} \et Q_v < \nicefrac{k}{2}) 
\end{align*}
This concludes the bound on $\E[R_v]$.
Let us turn to $\E[Q_v]$. Just note that:
\begin{align}
\E[Q_v] \ge \frac{k}{2} \cdot \prob(Q_v \ge \nicefrac{k}{2})
\end{align}
By summing over all nodes, we obtain:
\begin{align}
\E[Q] &\ge \sum_{v \in B} \E[Q_v] \ge \frac{k}{2} \Big(\sum_{v \in B} \prob(Q_v \ge \nicefrac{k}{2}) \Big) \label{eqn:eQ} \\
\E[\Delta] &\ge \sum_{v \in B} \E[R_v] > \frac{\alpha n}{k} \Big(\sum_{v \in B} \prob(\bar{S_v} \et Q_v < \nicefrac{k}{2}) \Big) \label{eqn:eDelta}
\end{align}
to which, by virtue of (\ref{eqn:sumup}), applies the constraint:
\begin{align}
\label{eqn:constr}
\Big(\sum_{v \in B} \prob(Q_v \ge \nicefrac{k}{2}) \Big) + \Big(\sum_{v \in B} \prob(\bar{S_v} \et Q_v < \nicefrac{k}{2}) \Big)  > |B|\frac{1}{2} = \frac{(1-\alpha)n}{2}
\end{align}
This constrained system gives the bound.
Indeed, by (\ref{eqn:eQ}), (\ref{eqn:eDelta}) and (\ref{eqn:constr}), it follows that if $\E[Q] < \frac{k}{2}\frac{(1-\alpha)n}{4} = \frac{(1-\alpha)nk}{8}$ then $\E[\Delta] > \frac{\alpha n}{k} \frac{(1-\alpha)n}{2} = \frac{\alpha(1-\alpha)n^2}{4k}$.
It just remains to set $\alpha$ and $k$ properly so to get the statement of the theorem.

Let $\alpha = \nicefrac{9}{10}$ and recall that $k=1/\epsilon$.
Then, first, $\frac{(1-\alpha)nk}{8} = \frac{nk}{80} = \frac{n}{80 \, \epsilon}$.
Second, (\ref{eqn:OPTbound}) gives $\E[\OPT] <\frac{(1-\alpha)^2n^2}{k} = \frac{n^2}{100 k} = \frac{\epsilon n^2}{100}$.
Third, $\frac{\alpha(1-\alpha)n^2}{4k} = \frac{9 n^2}{400 k} = \frac{9 \epsilon n^2}{400} > \E[\OPT] + \frac{ \epsilon n^2}{80}$.
The above statement hence becomes: if $\E[Q] < \frac{n}{80 \epsilon}$, then $\E[\Delta] > \E[\OPT]+ \frac{ \epsilon n^2}{80}$.
An application of Yao's minimax principle completes the proof.

As a final note, we observe that for every $c \ge 1$ the bound can be put in the form $\E[\Delta] \ge  c \cdot \E[\OPT] + \Omega(n^2\epsilon)$  by choosing $\alpha \ge c/(c+\nicefrac{1}{4})$.

\section{Supplementary Material for Section 7}
\label{s:ex}
We report the complete experimental evaluation of $\access$ including error bars (see the main paper for a full description of the experimental setting). The details of the datasets are found in Table \ref{t:datasets}.

\begin{table}[ht!]
\caption{
\label{t:datasets}
Description of the datasets.
}
\centering
\begin{tabular}{lccc}
\toprule
Datasets & Type & $|V|$ & \#Clusters \\
\midrule
captchas & Real & 244 & 69 \\
cora & Real-world & 1879 & 191 \\
gym & Real & 94 & 12 \\
landmarks & Real & 266 & 12 \\
skew & Synthetic & 900 & 30 \\
sqrt & Synthetic & 900 & 30 \\
\midrule
\end{tabular}
\end{table}

\begin{figure}
\caption{Clustering cost vs.\ number of queries.}
\newcommand{\mywidth}{0.46\linewidth}
\begin{subfigure}[b]{\mywidth}
\includegraphics[width=\linewidth]{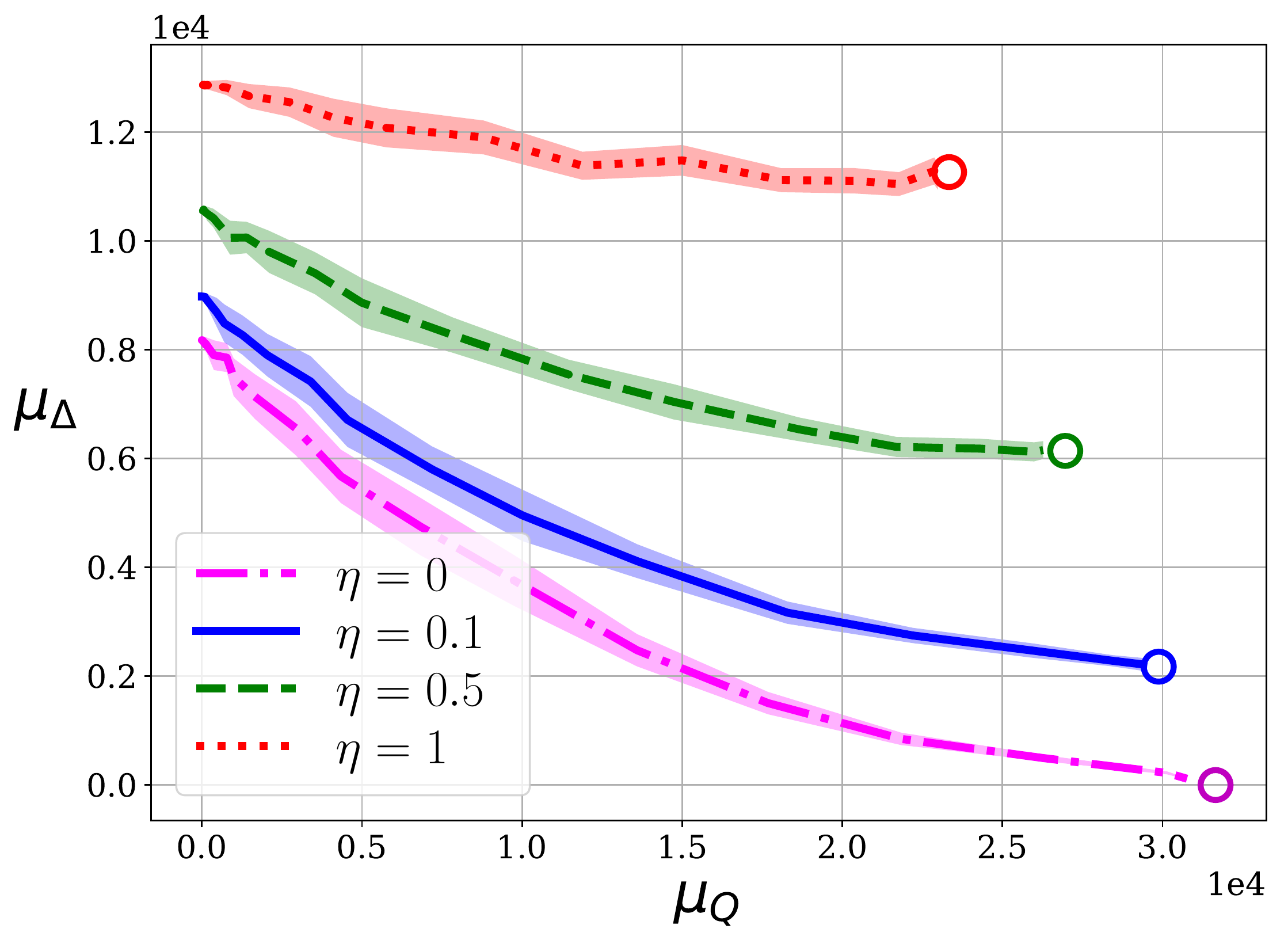}
\caption{skew.}
\end{subfigure}
\hspace*{20pt}
\begin{subfigure}[b]{\mywidth}
\includegraphics[width=\linewidth]{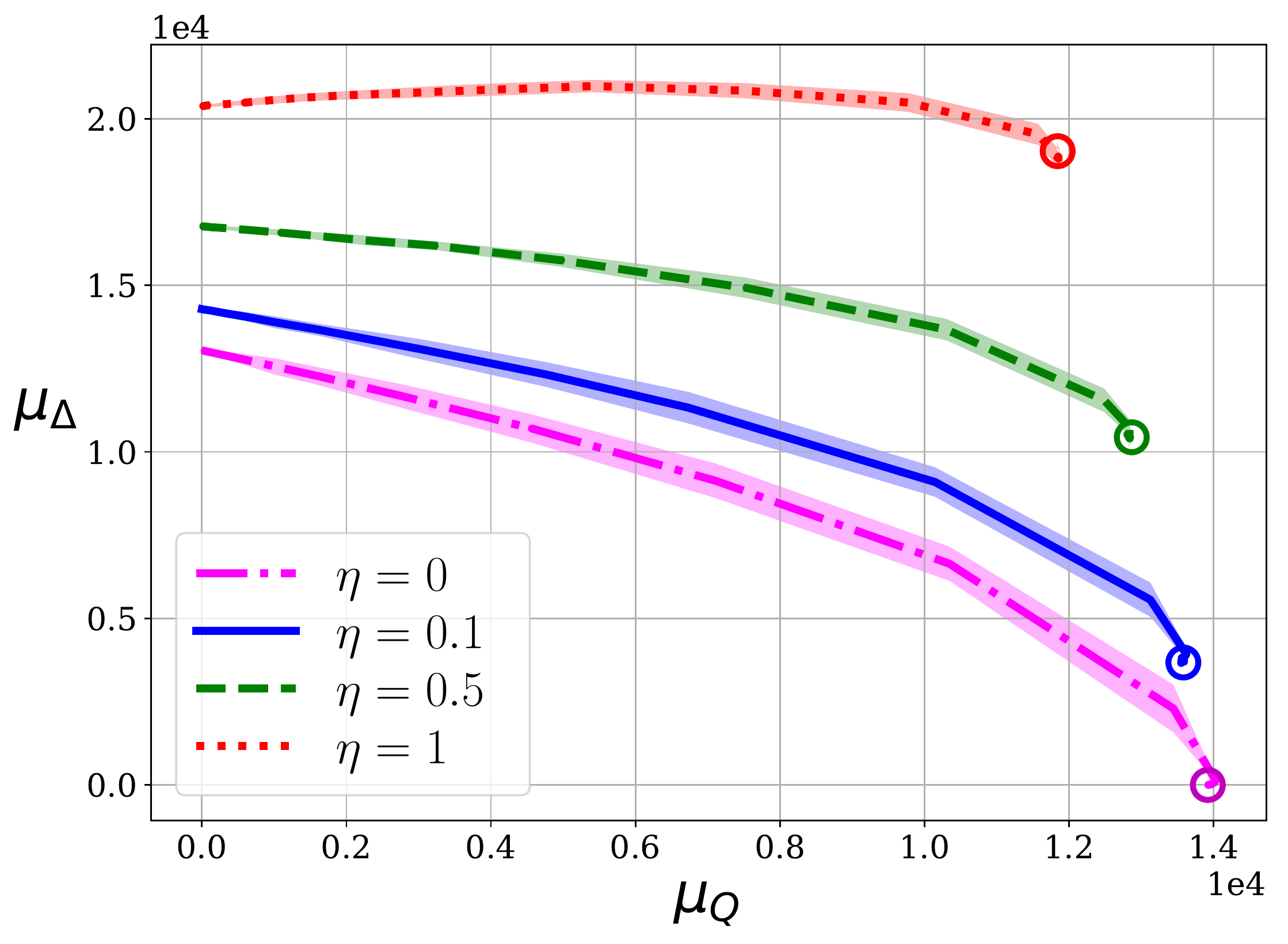}
\caption{sqrt.}
\end{subfigure}
\\%\hspace*{20pt}
\begin{subfigure}[b]{\mywidth}
\includegraphics[width=\linewidth]{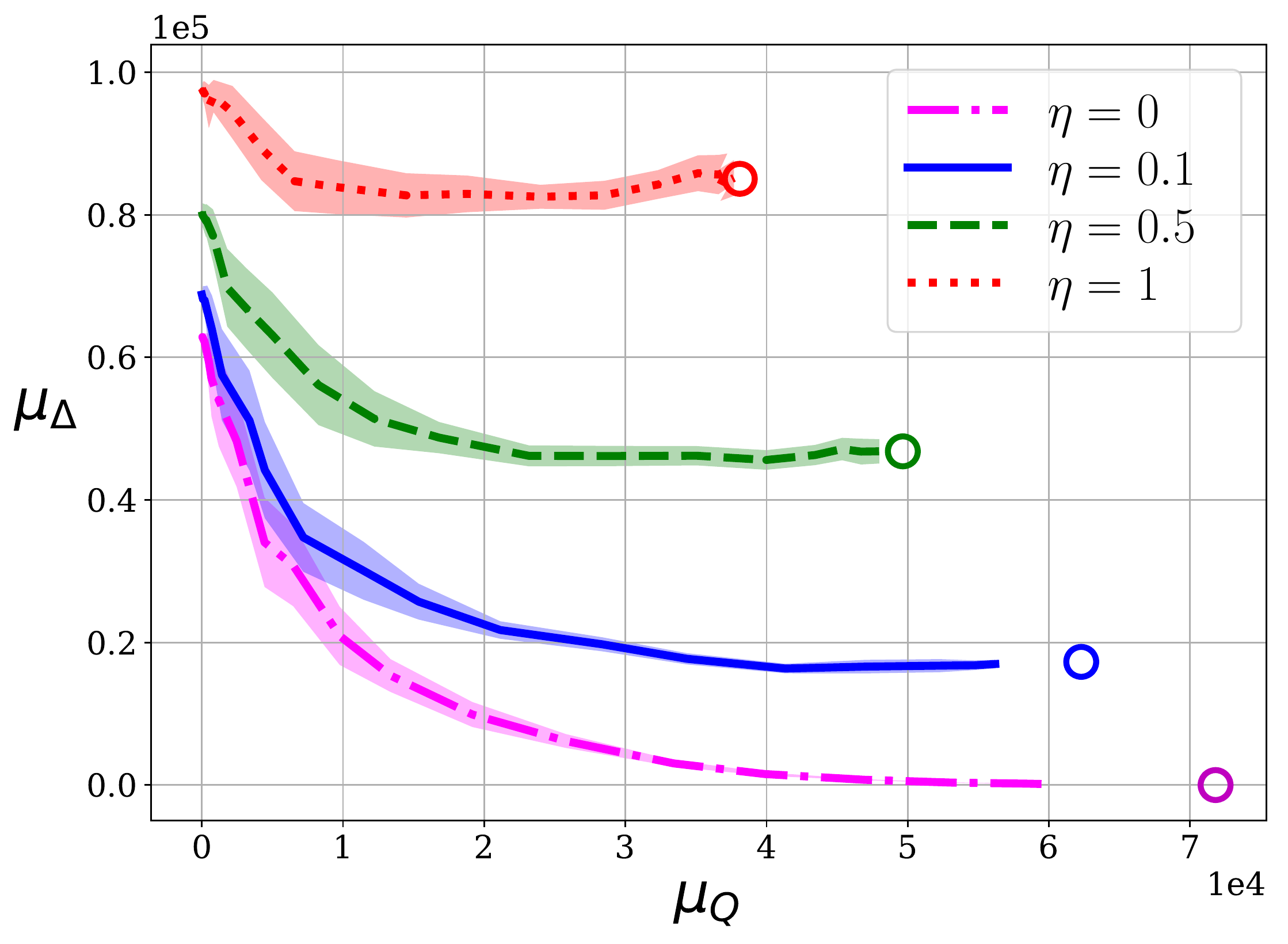}
\caption{cora.}
\end{subfigure}
\hspace*{20pt}
\begin{subfigure}[b]{\mywidth}
\includegraphics[width=\linewidth]{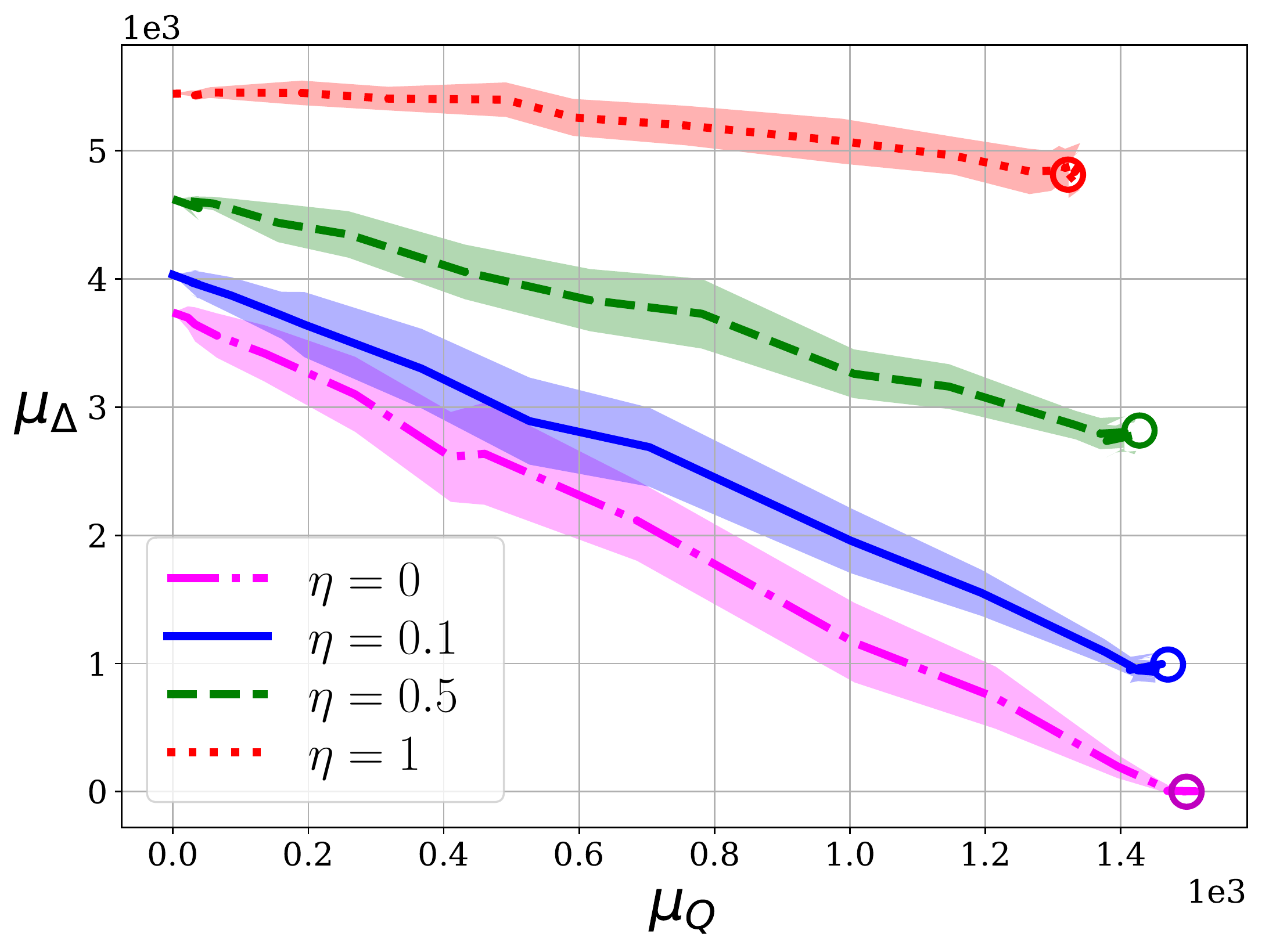} 
\caption{landmarks.}
\end{subfigure}
\\%\hspace*{20pt}
\begin{subfigure}[b]{\mywidth}
\includegraphics[width=\linewidth]{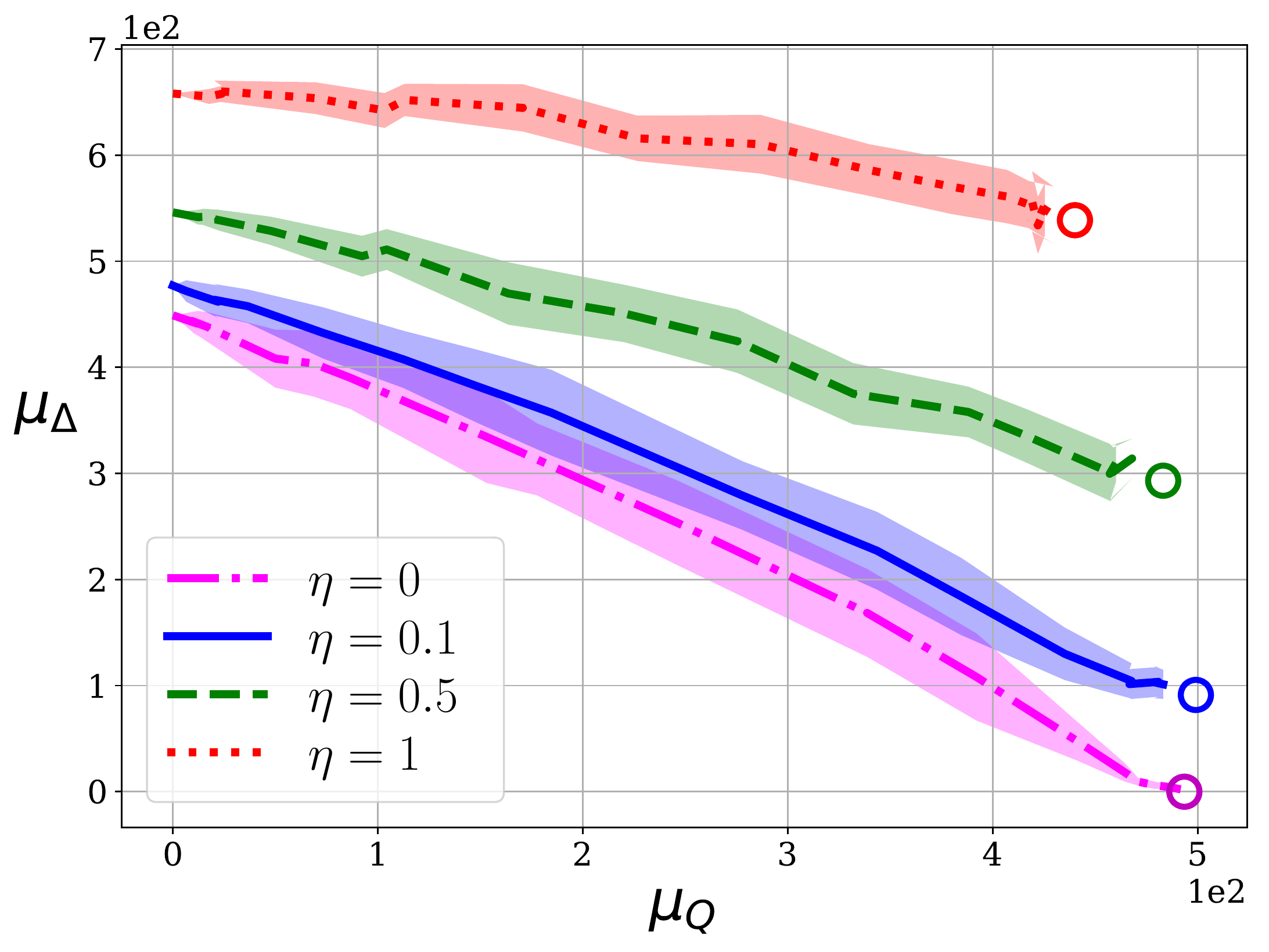}
\caption{gym.}
\end{subfigure}
\hspace*{20pt}
\begin{subfigure}[b]{\mywidth}
\includegraphics[width=\linewidth]{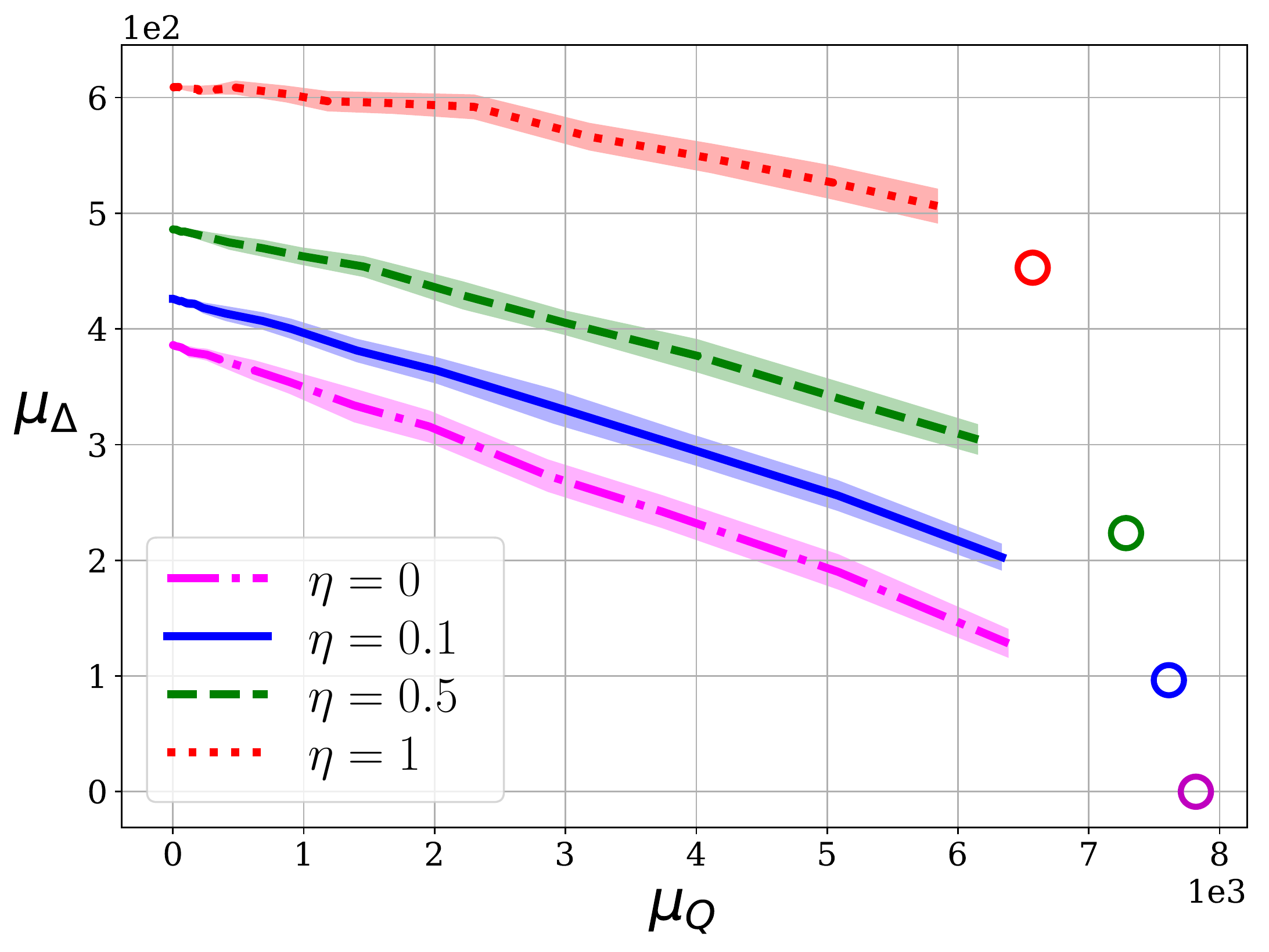} 
\caption{captchas.}
\end{subfigure}
\end{figure}

%Bib***************************************************************************
%\bibliographystyle{plainnat}
%\bibliography{biblio}

\end{document}